\documentclass[lettersize,journal]{IEEEtran}
\usepackage{amsmath,amsfonts}
\usepackage{algorithm}
\usepackage{algorithmic}
\usepackage{array}
\usepackage[caption=false,font=small]{subfig}
\usepackage{textcomp}
\usepackage{stfloats}
\usepackage{url}
\usepackage{verbatim}
\usepackage{graphicx}
\usepackage{cite}

\usepackage{amsmath,amsthm,amsfonts,bm}

\newtheorem{theorem}{Theorem}
\newtheorem{lemma}[theorem]{Lemma}
\newtheorem{definition}[theorem]{Definition}

\newtheorem{assumption}[theorem]{Assumption}
\newtheorem{condition}[theorem]{Condition}

\def\vf{{\bm{f}}}

\def\vu{{\bm{u}}}
\def\vv{{\bm{v}}}
\def\vw{{\bm{w}}}
\def\vx{{\bm{x}}}
\def\vy{{\bm{y}}}

\def\mI{{\bm{I}}}

\DeclareMathAlphabet{\mathsfit}{\encodingdefault}{\sfdefault}{m}{sl}
\SetMathAlphabet{\mathsfit}{bold}{\encodingdefault}{\sfdefault}{bx}{n}

\def\gD{{\mathcal{D}}}

\def\gF{{\mathcal{F}}}

\def\gH{{\mathcal{H}}}

\def\gL{{\mathcal{L}}}

\def\gR{{\mathcal{R}}}

\def\gX{{\mathcal{X}}}
\def\gY{{\mathcal{Y}}}

\def\sE{{\mathbb{E}}}

\def\sR{{\mathbb{R}}}

\DeclareMathOperator*{\argmax}{arg\,max}
\DeclareMathOperator*{\argmin}{arg\,min}

\renewcommand{\tilde}{\widetilde}
\renewcommand{\hat}{\widehat}

\newcommand{\vertiii}[1]{{\left\vert\kern-0.25ex\left\vert\kern-0.25ex\left\vert #1 
    \right\vert\kern-0.25ex\right\vert\kern-0.25ex\right\vert}}
\usepackage{bbm}
\usepackage[numbers, sort & compress]{natbib}
\usepackage{xcolor}
\usepackage{tabularray}
\usepackage{enumitem}
\usepackage{multirow}
\usepackage{booktabs}
\usepackage{hyperref}
\usepackage{fontawesome5}
\usepackage{makecell}
\usepackage{xcolor}
\usepackage{soul}
\usepackage[most]{tcolorbox}
\usepackage[font=footnotesize,justification=centering]{caption}
\usepackage{siunitx} 

\hypersetup{
  pdftitle={Forgettable Federated Linear Learning with Certified Data Unlearning},
  pdfauthor={Ruinan Jin, Minghui Chen, Qiong Zhang, and Xiaoxiao Li},
  pdfsubject={Accepted manuscript for IEEE Transactions on Neural Networks and Learning Systems. DOI: 10.1109/TNNLS.2026.3683398},
  pdfkeywords={federated learning, certified unlearning, machine unlearning, certified federated unlearning, foundation models}
}

\newcommand{\ours}{$\mathtt{F^2L^2}$}
\newcommand{\remove}{$\mathtt{FedRemoval}$}
\newcommand{\train}{$\mathtt{FLT}$}
\newcommand{\revision}[1]{{\color{black}#1}}

\newlength\myindent
\newcommand\bindent{%
  \begingroup
  \setlength{\itemindent}{\myindent}
  \addtolength{\algorithmicindent}{\myindent}
}
\newcommand\eindent{\endgroup}

\begin{document}

\title{Forgettable Federated Linear Learning with Certified Data Unlearning}

\author{Ruinan Jin, Minghui Chen, Qiong Zhang, and Xiaoxiao Li%
\thanks{This work was supported in part by the Natural Sciences and Engineering under Grant RGPIN-2022-0531 and in part by the Renmin University of China. (Corresponding author: Qiong Zhang.)}%
\thanks{This is the accepted manuscript of the article accepted for publication in \textit{IEEE Transactions on Neural Networks and Learning Systems}, doi: 10.1109/TNNLS.2026.3683398. This is not the final published version.}%
\thanks{\copyright{} 2026 IEEE. Personal use of this material is permitted. Permission from IEEE must be obtained for all other uses, in any current or future media, including reprinting/republishing this material for advertising or promotional purposes, creating new collective works, resale or redistribution to servers or lists, or reuse of any copyrighted component of this work in other works.}%
\thanks{Ruinan Jin, Minghui Chen, and Xiaoxiao Li are with the University of British Columbia, Vancouver, BC V6T 1Z4, Canada, and also with the Vector Institute, Toronto, ON M5G 0C6, Canada (e-mail: ruinanjin@alumni.ubc.ca; minghui.chen.research@outlook.com; xiaoxiao.li@ece.ubc.ca).}%
\thanks{Qiong Zhang is with the Renmin University of China, Beijing 100872, China (e-mail: qiong.zhang@ruc.edu.cn).}}

\IEEEpubid{}

\maketitle

\begin{abstract}
The advent of Federated Learning (FL) has revolutionized the way distributed systems handle collaborative model training while preserving user privacy. Recently, Federated Unlearning (FU) has emerged to address demands for the ``right to be forgotten'' and unlearning of the impact of poisoned clients without requiring retraining in FL. 
Most FU algorithms require the cooperation of retained or target clients (clients to be unlearned), introducing additional communication overhead and potential security risks. In addition, some FU methods need to store historical models to execute the unlearning process. These challenges hinder the efficiency and memory constraints of the current FU methods. Moreover, due to the complexity of nonlinear models and their training strategies, most existing FU methods for deep neural networks (DNN) lack theoretical certification. In this work, we introduce a novel FL training and unlearning strategy in DNN, termed Forgettable Federated Linear Learning (\ours{}). \ours{} considers a common practice of using pre-trained models to approximate DNN linearly, allowing them to achieve similar performance as the original networks via Federated Linear Training (\train). We then present \remove{}, a certified, efficient, and secure unlearning strategy that enables the server to unlearn a target client without requiring client communication or adding additional storage. We have conducted extensive empirical validation on small- to large-scale datasets, using both convolutional neural networks and modern foundation models. These experiments demonstrate the effectiveness of \ours{} in balancing model accuracy with the successful unlearning of target clients. \ours{} represents a promising pipeline for efficient and trustworthy FU. The code is available~\href{https://github.com/Nanboy-Ronan/2F2L-Federated-Unlearning}{here}.
\end{abstract}


\section{Introduction}
\label{sec:intro}

\newif\ifcondition
\conditiontrue 

\ifcondition
\IEEEPARstart{F}{ederated} Learning (FL) enables collaborative model training across geographically distributed clients by aggregating locally trained models without centralized data collection. 
However, this decentralization exposes the global model to risks~\citep{fang2022robust, bagdasaryan2020backdoor,jin2023backdoor}: malicious clients\footnote{We refer to clients whose data must be deleted as \textit{target clients} and the rest as \textit{retained clients}.} may poison their updates (e.g., distorting disease diagnoses in healthcare FL), or copyrighted data in client contributions may introduce legal liabilities. 
\emph{To mitigate these risks, FL systems must efficiently remove harmful contributions—a capability known as Federated Unlearning (FU)}~\citep{nguyen2022survey,nguyen2024empirical}.

Yet, existing FU methods, as detailed in supplementary material~\ref{related}, face critical limitations that hinder real-world adoption. Naive solutions~\citep{nguyen2022survey,nguyen2024empirical}, such as retraining without the client(s) to be \revision{unlearned}, are prohibitively expensive due to computational and communication costs, as well as client unavailability. 
\revision{As summarized in Table~\ref{tab:baseline} in supplementary material~\ref{related}, no existing FU strategy simultaneously avoids extra communication, target and retained client involvement, and historical model storage while providing certified unlearning guarantees. In particular, each method violates at least one of these requirements as detailed below:}

\begin{itemize}[leftmargin=*]
\item \textbf{Dependence on retained clients}: Methods like~\cite{parisi2019continual,nguyen2024empirical,halimi2022federated,liu2024privacy,zhong2025unlearning} require retained clients to participate in unlearning, which is often infeasible due to computational constraints and communication overhead.

\item \textbf{Reliance on target clients}: Approaches such as~\cite{wu2022federated,li2021anti,halimi2022federated,gu2024ferrari} demand that target clients perform customized optimizations (e.g., gradient ascent, Lipschitz updates) on private data and return updates—a security risk if clients are malicious and unlikely to cooperate post-attack.

\item \textbf{Storage inefficiency}: FedEraser~\citep{liu2021federaser} and variants store all intermediate global models, incurring prohibitive memory costs for large-scale deep neural networks (DNNs).

\item \textbf{Theoretical verification gaps}: The nonlinearity of DNNs complicates the theoretical verification of success of the unlearning process.
\end{itemize}

These limitations motivate our core research question:
\begin{center}
\begin{tcolorbox}
[colframe=black!50!white, colback=gray!5!white, sharp corners=all, boxrule=0.2mm, rounded corners=southeast, arc is angular, width=8.5cm]
\centering
\textit{How can we design a practical FU framework that ensures certified unlearning without requiring retained client cooperation, target client participation, and historical model storage?}
\end{tcolorbox}
\end{center}

We propose the Forgettable Federated Linear Learning (\ours{}) framework, which integrates a novel training pipeline (Federated Linear Training, \train{}) with a server-side unlearning mechanism (\remove{}). 
Inspired by the neural tangent kernel~\citep{jacot2018neural} and its extensions to pretrained models~\citep{arora2019harnessing}, \train{} employs a first-order linear approximation of deep neural networks (DNNs) around a pretrained model. 
This linearization transforms the loss function into a quadratic form, ensuring convexity and simplifying optimization compared to non-convex objectives. Despite its simplicity, \train{} achieves performance competitive with traditional FL methods by strategically linearizing around a well-initialized pretrained model.

Crucially, the linear model enables a closed-form unlearning solution: the optimal post-unlearning weights are just a Newton step away from the original weights. \remove{} exploits this property by computing the Newton update using only the target clients' final gradient contributions—already shared during training—eliminating the need for additional client cooperation or historical model storage.

We further prove that the weights obtained via \remove{} remain provably close to those from full retraining. Specifically, we derive an upper bound on the $L2$-norm discrepancy between the two, certifying the effectiveness of our unlearning process.

In summary, our approach yields four key advantages as illustrated in Fig.~\ref{fig:fu_property}.
\begin{figure}[ht]
\centering
\includegraphics[width=0.8\linewidth]{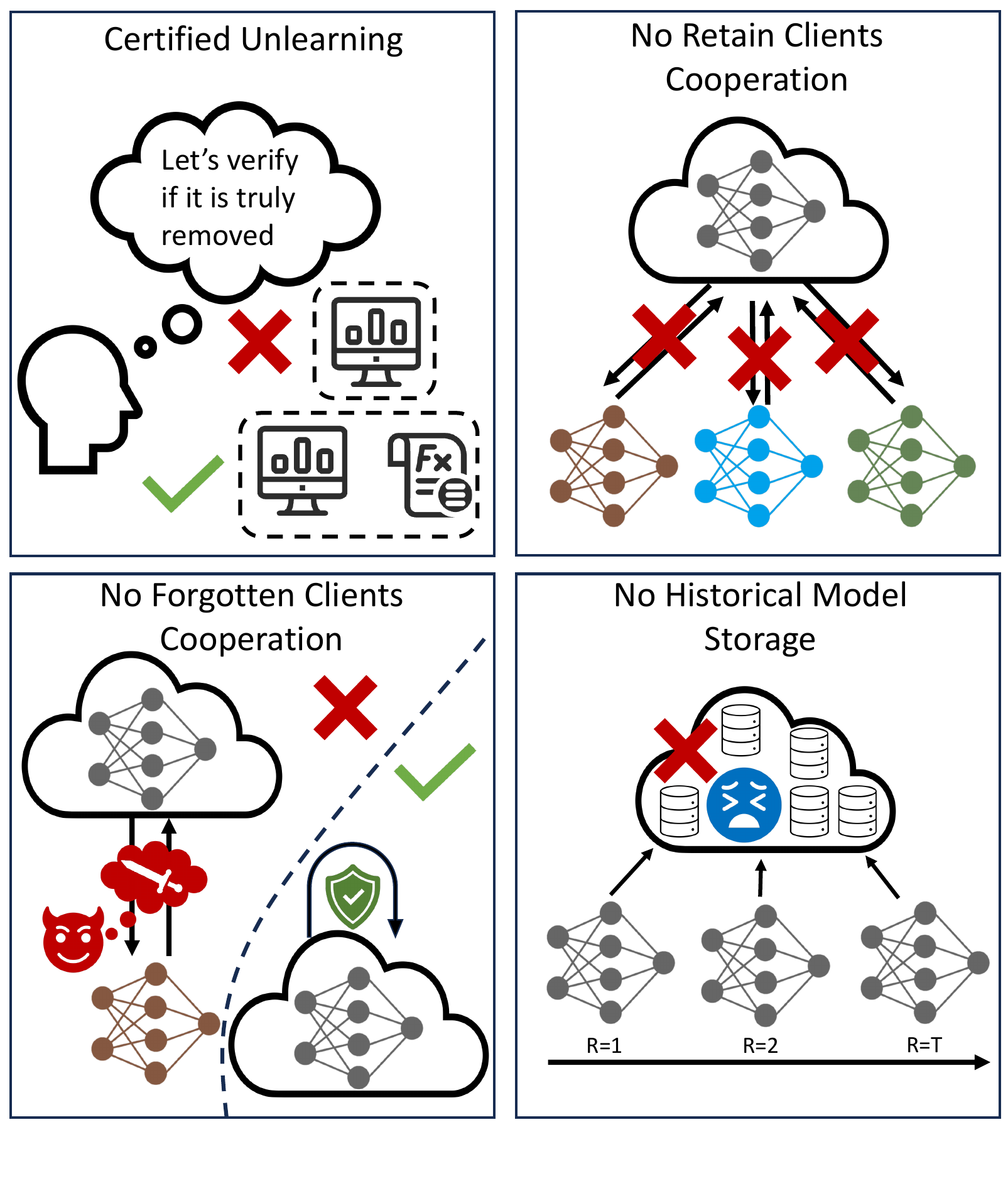}   
\caption{\textbf{Advantages of the proposed FU algorithm}. Our method is a certified unlearning approach without requiring retained client cooperation, target client participation, and historical model storage.}
\label{fig:fu_property}
\end{figure}

This paper advances federated unlearning through three key contributions:
\begin{itemize}[leftmargin=*]
    \item  \textbf{Novel FL framework}: We present \textbf{F}orgettable \textbf{F}ederated \textbf{L}inear \textbf{L}earning (\ours), the \textit{first} federated unlearning solution that combines \train{} (a distributed mixed-linear training approach) with \remove{} (a purely server-side unlearning mechanism). Specifically designed for FL systems, our framework eliminates both the need for client cooperation during unlearning and the storage overhead of historical models--addressing two fundamental limitations of existing methods.
    
    \item \textbf{Certified unlearning guarantees}: We establish rigorous theoretical foundations for federated unlearning by deriving a provable bound on the weight deviation between our approach and complete retraining from scratch. 
    
    \item \textbf{Comprehensive empirical evaluation}: Through extensive experiments on six diverse datasets and four model architectures (including foundation models like CLIP), we demonstrate \ours{}'s effectiveness in completely removing backdoor attacks while maintaining model performance. 
\end{itemize}

The rest of the paper is structured as follows. 
Sec.~\ref{preliminary} introduces the foundations of FL, unlearning under quadratic loss, and the notations used throughout the paper. Sec.~\ref{method:FFL} presents the complete \ours{} pipeline along with its theoretical properties. Sec.~\ref{exp} provides the experimental details, results, and discussions. Finally, we include a summary of notations, detailed theoretical analysis, and additional results in the supplementary material.

\else
\IEEEPARstart{F}{ederated} Learning (FL) facilitates collaborative model training across geographically dispersed data centers without centralizing private data. 
This approach aligns with privacy regulations such as the \textit{California Consumer Privacy Act}~\citep{ccpa}, the \textit{Health Insurance Portability and Accountability Act}~\citep{act1996health}, and the \textit{General Data Protection Regulation}~\citep{gdpr}. The most well-known FL methods, including FedAvg~\citep{mcmahan2017communication}, FedAdam~\citep{reddi2020adaptive}, and similar frameworks, use iterative optimization algorithms to enable concurrent model training among clients. 
In each round, clients perform stochastic gradient descent (SGD) over multiple steps and then send their model weights to server for aggregation.

Due to the intrinsic design of FL, which conceals data of private clients from the server, the aggregated model may inadvertently incorporate noisy or even malicious inputs from untrustworthy clients~\citep{fang2022robust, bagdasaryan2020backdoor}. 
For example, in a healthcare FL network, an attacker might manipulate their local model to distort predictions for rare disease markers, compromising the accuracy of the aggregated server model.
\textit{Alternatively}, the model may incorporate data containing copyrighted material, introducing legal risks to developers.
Retraining the model without the untrustworthy clients might address the issue;
however, this approach is often challenging due to the high financial and time costs, and remaining clients may not be available for additional rounds of training (e.g., due to loss of financial support).
In such situations, it is crucial that the FL organizers have the ability to efficiently remove harmful client contributions.
This requires an efficient and effective Federated Unlearning (FU) method, which involves removing specific traces of training data directly from the server model, ensuring its integrity and compliance with legal and ethical standards~\citep{liu2020federated, halimi2022federated}.

FU is a rapidly growing area of research, with many methods developed in recent years~\citep{nguyen2022survey,nguyen2024empirical}. 
However, a \emph{significant limitation} of the existing FU methods is their impracticality in real-world applications due to their reliance on retained clients (whose data must be preserved) or target clients (whose data must be unlearned).
Specifically, approaches proposed by~\cite{parisi2019continual,nguyen2024empirical,halimi2022federated,liu2024privacy} often require the cooperation of retained clients during unlearning, but \textit{ requesting the cooperation of clients is often impractical and inefficient} in practice. 
This is because retained clients may lack sufficient computational resources for unlearning and may struggle with the communication overhead involved.
Other approaches, such as those of~\cite{wu2022federated,li2021anti,halimi2022federated}, require the participation of target clients, asking them to perform a stochastic gradient ascent on their private data and return the model to the server.
However, \textit{this may introduce further security risks} if the target client is untrustworthy and possesses malicious inputs.
Such clients are also unlikely to voluntarily provide their poisoned data for post-attack clean-up.
\emph{Another limitation} of existing methods, such as FedEraser~\citep{liu2021federaser} and its variants, is memory inefficiency.
They require storing trained models for each global epoch, significantly increasing storage costs, which is especially impractical when model sizes are large and server memory is limited. 
\emph{Additionally}, verifying the success of the unlearning process theoretically presents challenges due to the complexity of nonlinear deep neural networks (DNNs).

In summary, the methods reviewed in the previous paragraphs and are outlined in Table~\ref{tab:baseline} show several notable limitations.
They often require the participation of retained or target clients, leading to substantial obstacles in operational efficiency, security, and privacy compliance. 
In addition, they typically lack mechanisms for certified data unlearning, a crucial requirement for regulatory adherence. 
These challenges highlight the need for a new FU strategy that is more flexible for client dependency, efficient in communication and storage, and theoretically certifiable. Motivated by these limitations, we ask the following research question that guides this work:

\begin{center}
\begin{tcolorbox}
[colframe=black!50!white, colback=gray!5!white, sharp corners=all, boxrule=0.5mm, rounded corners=southeast, arc is angular, width=8cm]
\centering
\textit{How can we design a federated unlearning framework that ensures certified unlearning without requiring client cooperation, forgotten client participation, or historical model storage?}
\end{tcolorbox}
\end{center}

\revision{In response to this challenge, we introduce the Forgettable Federated Linear Learning (\ours{}) framework, which incorporates novel certified unlearning strategies, \remove{}, supported by a new training pipeline called Federated Linear Training (\train{}).} 
Drawing inspiration from the neural tangent kernel~\citep{jacot2018neural} and its extension to pretrained models~\citep{arora2019harnessing}, we use a first-order linear approximation of deep neural networks (DNNs) around the pretrained model. 
By using a linear model, the loss becomes quadratic, making it convex and easier to minimize compared to non-convex losses.
In \train{}, we employ the commonly used FL aggregation method (FedAvg~\citep{mcmahan2017communication}) for optimization within the federated learning (FL) setting. Empirical results show that this training strategy achieves model performance comparable to traditional training methods. 
Additionally, the quadratic loss enables a closed-form relationship between the optimal weights before and after unlearning, showing that the optimal weight post-unlearning is just a Newton step away from the original optimal weight.
When unlearning is requested, \remove{} performs unlearning by computing the Newton update on the server, using only the gradients uploaded by the clients during their last epoch in FedAvg.
This allows for efficient removal of client data without requiring their further participation. 
\textit{Theoretically}, we provide an upper bound on the difference between the model weights from our approach and those from retraining the model from scratch, thereby certifying the unlearning process.

In summary, this paper makes the following contributions:
\begin{itemize}
    \item \textbf{Technically}, we propose \ours{}, a novel strategy to train FL models and enable clients to unlearn. \ours{} consists with an effective training method, \train{}, and removal method, \remove{}. The proposed \remove{} operates solely on the server, eliminating the need for client cooperation and requiring no additional storage. Based on a pre-trained model, this is achieved by \train{} that adapts mix-linear training in the distributed setting, and our innovative method to approximate the unlearning hessian matrix. 
    \item \textbf{Theoretically}, we present the certified unlearning by establishing an upper bound on the discrepancy between the model weights obtained through our proposed approach and those derived from retraining the model from scratch.
    \item \textbf{Empirically}, we conduct experiments on \textit{six} datasets and four model architectures, including popular foundation models like CLIP. Our results validate the performance of \ours{} by effectively unlearning backdoor attacks on targeted clients.
\end{itemize}

This paper is structured as follows. Sec.~\ref{related} reviews existing FU algorithms. Sec.~\ref{preliminary} introduces the foundations of FL, unlearning under quadratic loss, and the notations used throughout the paper. Sec.~\ref{method:FFL} presents the complete \ours{} pipeline along with its theoretical properties. Sec.~\ref{exp} provides the experimental details, results, and discussions. Finally, we include a summary of notations, detailed theoretical analysis, and additional results in the supplementary material.

\fi

\section{Preliminaries}
\label{preliminary}
This section begins by defining the problem. For easy reference, all the notation used in this manuscript are summarized in the supplementary material A.

\subsection{Standard FL Algorithm}
\label{sec:fedavg}
We first review the classical FL approach, FedAvg~\citep{mcmahan2017communication}, for classification. 
Consider a $K$ class classification task with feature space $\gX \subset \mathbb{R}^{d}$ and label space $\gY = [K] = \{1, 2, \ldots, K\}$. Let $\vf: \gX \to \mathbb{R}^{K}$ be a $K$ class classifier parameterized by weight $\vw$ such that
\[
p(y = k | \vx) = \sigma(f_k(\vx; \vw))
\]
where $f_k(\vx; \vw)$ is the $k$-th element of $\vf(\vx; \vw)$, and $\sigma(f_k) = \exp(f_k) / \sum_{k' = 1}^{K} \exp(f_{k'})$ is the softmax function.

In a classical machine learning (ML) problem, the weight $\vw$ of the classifier is learned based on a training dataset $\gD$.
Under the FL setting, the training dataset is not available on a single device and is partitioned over $C$ clients. 
Let $\gD_{c} = \{(x_c^i, y_c^i)\}_{i=1}^{n_c}$ be the training set on the $c$-th client. 
We consider the case where $\gD_{c}$ are independent identically distributed (IID) samples from the distribution $p(\vx;y)$.
We denote the full training set by $\gD= \{\gD_{1},\ldots,\gD_{C}\}$ and the total sample size $n=n_1+\ldots+n_C$.

Since each client $c\in[C]$ only has access to $\gD_c$ of size $n_c$, then the local empirical risk on $c$-th client becomes
\begin{equation*} 
    \gL(\vw;\gD_{c})=\frac{1}{n_c}\sum_{i=1}^{n_c}\ell(\vf(\vx_c^i;\vw),y_c^i),
\end{equation*}
where $\ell:\mathbb{R}^K\times\mathcal{Y}\to\mathbb{R}_{+}$ is a proper loss function such as the \emph{cross-entropy (CE) loss} or the \emph{mean squared error (MSE) loss}.
As a result, the overall loss based on the entire training dataset is $\gL(\vw;\gD)=\sum_{c=1}^C p_{c}\gL(\vw;\gD_{c})$ where $p_c = n_c/n$.

Since raw data sharing is neither feasible nor private-preserving under the FL setting, the goal of FedAvg is to let $C$ clients collaboratively train a global classification model $\vf$ that minimizes the overall loss $\gL(\vw;\gD)$, without sharing the local datasets $\{\gD_c\}_{c=1}^{C}$.
To achieve this goal, the FedAvg employs a server to coordinate the following iterative distributed training:
\begin{itemize}
    \item in the $r$-th global round of training, the server broadcasts its current model weight $\vw_s^{r-1}$ to all the clients;
    \item each client $c$ initialize the model with current server model weight $\vw_c^{r,0}=\vw_s^{r-1}$ and performs $M$ local step updates:
    \begin{equation*}
        \vw_c^{r, m}\leftarrow\vw_c^{r, m-1}-\eta_l\cdot g_{c}^{r,m-1}
    \end{equation*}
    for $m\in[M]$ where $\eta_l$ is the local learning rate and  $g_{c}^{r,m}$ is unbiased estimate of $\nabla\gL(\vw_c^{r, m};\gD_{c})$ based on the mini-batch at the $m$-th local step in $r$-th round. 
    \item each client sends $\vw_c^{r, M}$ back to the server and the server aggregates the updates from all clients to form the new server model weight:
    $\vw_s^{r} = \sum_{c=1}^Cp_c\vw_c^{r, M}$.
\end{itemize}
The above iterative procedure is performed for $R$ global rounds until some convergence criterion has been met.
After the training, given the learned model parameter $\vw_s^{R}$, the predicted label given a new observation $\tilde \vx$ can be obtained via $\hat y = \mathop{\argmax}_{k\in [K]} f_{k}(\tilde\vx;\vw_s^{R})$.

\subsection{Centralized Unlearning under Quadratic Loss}
Recall that $\gD$ represents the training dataset used to train a model in the centralized setting.
Let the forget set $\gD_{f}\subset \gD$ be a subset of the training data whose information must be erased from the trained model.
An unlearning procedure operates on the trained model weights to generate a new set of weights that are indistinguishable from those obtained by retraining the model from scratch on the remaining dataset $\gD^{-}=\gD\backslash\gD_{f}$.

An important example of removal in a centralized setting is the linear regression problem~\citep{guo2020certified}, where the loss function is quadratic in model parameters:
\[\gR(\vw;\gD) = |\gD|^{-1}\sum_{(\vx,y)\in \gD} \|y - \vw^{\top} \vx\|^2.\]
Let $\hat \vw = \argmin_{\vw} \gR(\vw;\gD)$ and $\vw^{-} = \argmin_{\vw} \gR(\vw;\gD^{-})$.
Since $\gR(\vw;\gD^{-})$ is quadratic in $\vw$ and is hence convex, we must have $0 = \nabla \gR(\vw^{-};D^{-})$ and
the Taylor expansion of the gradient at $\hat \vw$ gives
\begin{equation*}
0 = \nabla \gR(\vw^{-};\gD^{-})= \nabla \gR(\hat \vw; \gD^{-}) + \nabla^2 \gR(\hat \vw; \gD^{-}) (\vw^{-} - \hat{\vw}).   
\end{equation*}
As a result,
\begin{equation}
\label{eq:removal}
\begin{split}
\vw^{-} &= \hat{\vw} - \left\{\nabla^2 \gR(\hat \vw; \gD^{-})\right\}^{-1}\nabla \gR(\hat \vw; \gD^{-})\text{\revision{.}} \\ 
\end{split}
\end{equation}

Building upon~\eqref{eq:removal}, it becomes unnecessary to retrain the entire model. 
Instead, one can perform a simple Newton step to obtain the optimal weight based on the remaining dataset.
Unfortunately, for more general ML models with nonquadratic losses,~\eqref{eq:removal} is not accurate if applied directly.

We generalize this concept to unlearning under FL by first linearizing the DNN in the parameter space. 
The primary challenge in applying~\eqref{eq:removal} is to calculate Hessian $\nabla^2\gR$ and its inverse, given the significantly large dimensions in DNNs. 
In the FL setting, where the remaining dataset is distributed across different clients, this computation becomes even more expensive, especially if communication among clients is required. 
Therefore, overcoming this computational challenge is crucial to achieving effective unlearning.

\begin{figure*}
    \centering
    \includegraphics[width=0.9\textwidth]{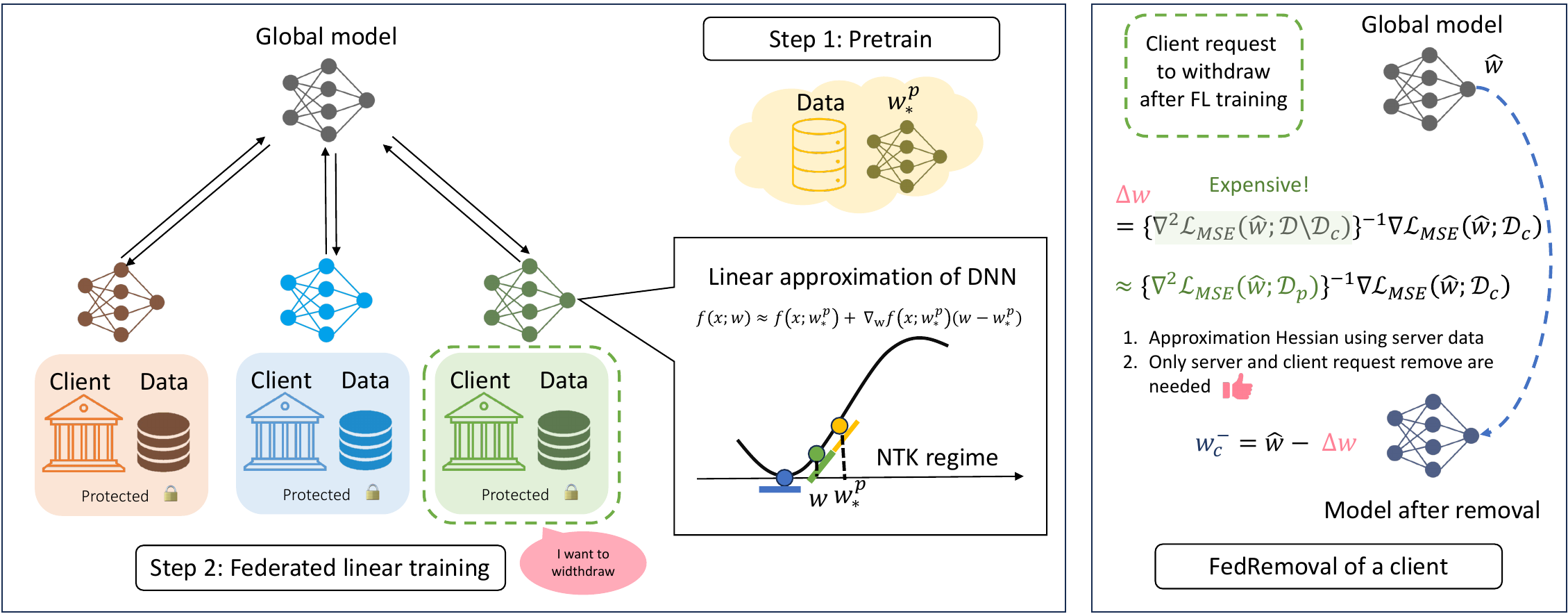}
    \captionsetup{justification=raggedright}
    \caption{\textbf{Illustration of the problem setting and our proposed Forgettable Federated Linear Learning (\ours{}) framework.} Our proposed algorithm enables the FL system to seamlessly remove a client's information from a linear model by utilizing a simple Newton's step. To achieve this, \train{} uses a linearized DNN with FedAvg during FL. When unlearning is required for a specific client, we perform a straightforward Newton step on the model weights of the linearized DNN. To mitigate the additional communication costs associated with FL, we employ an efficient Hessian approximation in the removal step. This approach ensures both efficient unlearning and communication in FL.}
    \label{fig:enter-label}
\end{figure*}

\section{\ours{} Pipeline}
\label{method:FFL}
In our \ours{} pipeline, we first introduce a novel FL training strategy called Federated Linear Training (\train). 
We then develop a tailored certified removal strategy with theoretical guarantees, termed \remove{}. 
These training and removal strategies are coupled to ensure efficient and certified data removal. 
\train{} maintains the same communication overhead as FedAvg (or any other federated aggregation strategy the user chooses), while \remove{} is executed solely on the server, requiring no additional communication with clients. 

\subsection{Federated Linear Training (\train)}
\begin{algorithm}[htbp]
\begin{algorithmic}
\bindent
    \STATE{\textbf{Input}: initial weight $\vw_s^{0}$, learning rate $\eta_l$, clients' dataset $\{\gD_{c}\}_{c=1}^{C}$, $p_c =n_c/n$}
    \FOR {$r = 0,1,\ldots, R-1$}
    \STATE Broadcast server weight $\vw_c^{r,0} = \vw_s^r$ to $C$ clients\;
    \FOR {$c = 1,\ldots, C$}
    \FOR {batch $b_m$, $m=1,\ldots, M$}
    \STATE{$g_{c}^{r,m} = (2|b_m|)^{-1}\sum_{i \in b_m}\nabla \|\tilde\vf(\vx_c^i;\vw_{c}^{r,m})-\vy_c^i\|^2+ \mu\vw_{c}^{r,m}$}
    \STATE{$\vw_{c}^{r,m}\leftarrow\vw_{c}^{r,m-1} - \eta_{l} g_{c}^{r,m-1}$}
    \ENDFOR
    \STATE{Transmit the client weight $\vw_{c}^{r,M}$ and gradients $\nabla \gL_{\text{MSE}}(\vw_{c}^{r,M};\gD_{c})$ to server}
    \ENDFOR
    \STATE{Aggregate on server via $\vw_{s}^{r+1} \leftarrow \sum_{c=1}^{C} p_{c}\vw_{c}^{r,M}$}
    \ENDFOR
    \RETURN{$\vw_{s}^{R}$}
\eindent
\caption{Proposed \train{} pipeline.}
\label{alg:learning}  
\end{algorithmic}
\end{algorithm}
Motivated by the fact that unlearning is much easier under a quadratic loss function compared to a general loss function, we design a novel FL method that leverages a quadratic loss function during training. 
At a high level, we propose using a linear approximation of a DNN based on the first-order Taylor expansion:
\begin{equation}
\label{eq:linear_approximation}
\vf(\vx;\vw) \approx \vf(\vx; \vw_{*}^{p}) + \langle \nabla_{w} \vf(\vx;\vw_{*}^{p}), \vw  - \vw_*^{p}\rangle\text{\revision{.}} 
\end{equation}
at some $\vw_{*}^{p}$ when $\vw$ is close to $\vw_{*}^{p}$.
The right-hand side of~\eqref{eq:linear_approximation} is linear in the model weight $\vw$, enabling unlearning via~\eqref{eq:removal} under the quadratic MSE loss.

To train the linear model, we first find a suitable initial value $\vw_{*}^{p}$ and then perform the linear approximation. 
Since at least one trustworthy client will not remove its dataset from the FL system, we use their data to determine $\vw_{*}^{p}$. 
For simplicity, we treat this client as the server and denote its local dataset as $\gD_{p} = {(\vx_p^i, y_p^i)}_{i=1}^{n_p}$ where $n_p$ the number of training samples on the server.
\textit{This dataset will not be shared with any other FL participants, thus avoiding additional privacy leakage beyond what is typical in FedAvg.} 
The server will then execute \remove{}, while the remaining $C$ clients, each with its own dataset $\gD_{c}$, train the linear approximation model.

\noindent \textbf{Initial $w_{*}^p$.} The weight $w_{*}^p$ is crucial to provide a good starting point for linear approximation in the Taylor expansion. 
It can be obtained by using either the pre-trained model on publicly available data or fine-tuning the pre-trained FM on $\gD_{p}$. 
For example, we can either use pre-train models on ImageNet or apply popular FMs like CLIP~\citep{radford2021learning}.
In both ways, the weight of the learned model is given by
$$\vw_{*}^{p} = \argmin_{\vw} \gL(\vw;\gD_{p}) = \underset{\vw}{\argmin} \; n_p^{-1}\sum_{i=1}^{n_p}\ell(\vf(\vx_p^i;\vw),y_p^i).$$

\noindent \textbf{Linear expansion at $w_{*}^p$.} Given that \(\vw_{*}^{p}\) is close to the optimal model weight based on the full dataset \(\gD\), we perform a first-order Taylor expansion at \(\vw_{*}^{p}\) as shown in~\eqref{eq:linear_approximation}. This yields the approximated linear model:
\begin{equation}
    \label{eq:mixed-linear-model}
    \tilde\vf(\vx; \vw) = \vf(\vx; \vw_*^{p}) + \langle \nabla_{\vw} \vf(\vx; \vw_*^{p}), \vw - \vw_*^{p} \rangle.
\end{equation}
Note this approximation is equivalent to a kernel predictor with a kernel known as the neural tangent kernel (NTK)~\citep{jacot2018neural}, defined as:
\[
k(\vx, \vx') = \nabla_{\vw} \vf(\vx; \vw_*^{p}) \nabla_{\vw}^{\top} \vf(\vx; \vw_*^{p}),
\]
which defines a neural tangent space in which the relationship between weights and functions is linear. 
As the width of the network approaches infinity,~\eqref{eq:linear_approximation} becomes exact and remains valid throughout the training.

To train the linear model $\tilde{\vf}$ on $\{\gD_{c}\}_{c=1}^{C}$, we use the following local objective function
\begin{equation*}
    \gL_{\text{MSE}}(\vw;\gD_{c}) = \frac{1}{2n_c} \sum_{i=1}^{n_c} \|\tilde\vf(\vx_c^i;\vw)-\vy_c^i\|^2 + \frac{\mu}{2}\|\vw\|^2.
\end{equation*}
where $\vy_c^i$ is the one-hot representation of $y_i^{c}$ and $\gL_{\text{MSE}}$ is the MSE loss.
The overall loss then becomes
\begin{equation}
\label{eq:overall_loss}
\gL_{\text{MSE}}(\vw;\gD) =\sum_{c=1}^{C} p_{c}\gL_{\text{MSE}}(\vw;\gD_{c})\text{\revision{,}}
\end{equation}
and the goal of the linear training step is to find
\begin{equation}
\label{eq:global_optimal}
\vw^* = \argmin \gL_{\text{MSE}}(\vw;\gD)
\end{equation}
without the need to communicate the raw data across clients.

The MSE loss in~\eqref{eq:overall_loss} corresponds to the standard loss for a ridge regression problem, and it has a closed-form minimizer:
\[
\begin{split}
\vw^* =&~\{\nabla_{w} \vf(\gD_c;\vw_*^{p})\nabla_{w}^{\top} \vf(\gD_c;\vw_*^{p}) + \mu \mI\}^{-1}\\
&\times \nabla_{w} \vf(\gD_c;\vw_*^{p})\{\vf(\vx; \vw_*^{p}) - \nabla_{\vw}^{\top} \vf(\vx; \vw_*^{p}) \vw_*^{p}-\vy_c\}.    
\end{split}
\]
The $L_2$ regularization is used in~\eqref{eq:overall_loss} to ensure $\vw^*$ is well defined.
\emph{However, the closed-form minimizer involves the inversion of $\{\nabla_{w} \vf(\gD;\vw_*^{p})\nabla_{w}^{\top} \vf(\gD;\vw_*^{p}) + \mu \mI\}^{-1}$, which is typically high-dimensional and computationally infeasible.}
To address this, we use the FedAvg algorithm described in Sec.~\ref{sec:fedavg} as an alternative to minimize to minimize the overall empirical loss~\eqref{eq:overall_loss}. 
The proposed training pipeline is outlined in Algorithm~\ref{alg:learning}.

Our proposed \train{} can be easily combined with foundation models (FM) without additional cost using the technique of linear probing~\citep{tsouvalas2023federated}.
Let us denote the model parameters as $\vw = (\theta, \phi)$, where $\theta$ represents the fixed parameters of the pre-trained FM encoder, and $\phi$ represents the linear layer. During \train, $\theta$ remains fixed, and only $\phi$ is updated. 
Consequently, the local loss becomes:
\begin{equation*} 
    \gL(\phi; \gD_{c}) = \frac{1}{n_c}\sum_{i=1}^{n_c}\ell\left(\vf(\vx_c^i; \theta, \phi), y_c^i\right),
\end{equation*}
and the rest of the algorithms can be applied to the loss function that only depends on the weights of the linear layer.

For the proposed \train{} algorithm, we have the following contraction property.
\begin{lemma}[Contraction under FedAvg algorithm (Informal)]
\label{lemma:fedavg_progress}
Suppose the local loss $\gL_{c}(\vw) = \gL_{\text{MSE}}(\vw;\gD_{c})$ is $\beta$-smooth and $\mu$-strongly convex. 
Let $\vw_c^*=\argmin \gL_{c}(\vw)$, $\vw^* = \argmin \gL(\vw)$ where $\gL(\vw) = \sum_{c=1}^{C}p_{c}\gL_{c}(\vw)$.
Let $\sigma^2$ be the upper bound of the variance of the stochastic gradient $g_{c}^{r,m}$. 
Let $\vw_{s}^{r}$ be the output of Algorithm~\ref{alg:learning} in $r$th iteration with step size $\eta_l \leq (\beta M)^{-1}$, then we have
\[\sE\|\vw_s^{r}-\vw^*\|^2\leq \left(1-\frac{\mu}{4\beta M}\right)^{r}\sE\|\vw_{s}^{0}-\vw^*\|^2 + \Delta\] where $\sigma_{\gL} = \sum_{c} p_{c}\{\gL_{c}(\vw^*) -\gL_{c}(\vw_c^*)\}$ and $\Delta = 4\sigma_{\gL}/\mu + 120\sigma^2/(\mu \beta M)$. 
\end{lemma}
This lemma implies that our \train{} algorithm converges linearly to the true optimum, with a bias term $\Delta$ introduced by stochastic gradient descent, consistent with the general case.

\subsection{Proposed Removal (\remove{})}
After training the model $\tilde\vf(\vx; \vw^*)$, the client $c$ can request to unlearn its training dataset $\gD_{c} \subseteq \gD_{f}$ from the model. Recall that $\gD^{-} = \gD\backslash\gD_{f}$.
Employing the quadratic unlearning approach described in~\eqref{eq:removal}, we derive the model weight based on the remaining dataset as follows:
\begin{equation*}
\vw^{-} = \vw^* - \underbrace{\left\{\nabla^2 \gL_{\text{MSE}}(\vw^*;\gD^{-} )\right\}^{-1}\nabla \gL_{\text{MSE}}(\vw^*;\gD^{-} )}_{(*)}.
\end{equation*}

Here, $\gL_{\text{MSE}}(\vw;\gD^{-}) = \sum_{c'\neq c} \tilde p_{c'}\gL_{\text{MSE}}(\vw;\gD_{c'})$, where $\tilde p_{c'} = n_{c'}/(n-n_c)$. 
However, this step necessitates Hessian inversion, whose computational cost is cubic in the number of parameters. 
Consequently, such a computation becomes infeasible for DNNs with millions of parameters.
In the centralized setting, this computational challenge can be addressed~\citep{golatkar2021mixed} by using SGD to find the minimizer of
\begin{align*}
\gF(\vv) = &~\frac{1}{2}\vv^{\top}\nabla^2\gL_{\text{MSE}}(\vw^*;\gD^{-} )\vv - \nabla^{\top} \gL_{\text{MSE}}(\vw^*;\gD^{-} )\vv \\
=&~\frac{1}{2(n-n_c)}\underbrace{\sum_{\vx\in \gD^{-}} \|\nabla_{w}\vf(\vx;\vw_*^p)\vv\|^2}_{(A)} + \frac{\mu}{2}\|\vv\|^2\\
&- \underbrace{\nabla^{\top} \gL_{\text{MSE}}(\vw^*;\gD^{-} )\vv}_{(B)}.
\end{align*}
This objective is used because $(*)$ is the unique global minimizer of $\gF(\vv)$, and SGD only involves the computation of gradients--a process whose cost scales linearly with the number of parameters, making it significantly more cost-effective than directly computing $(*)$.
\emph{This approach suffers from high communication cost in the FL setting as we detail below.}

The numerical solver requires the evaluation of $\gF(\vv)$ to use autograd in PyTorch. 
The term (B) is easy to evaluate as long as the gradient term $\nabla^{\top} \gL_{\text{MSE}}(\vw^*;\gD^{-})$ is known. 
This is straightforward and incurs no additional cost, as the gradients are sent to the server as part of FedAvg. 
The most computationally expensive part is (A) in the FL setting. 
Note that in a single SGD step, \textit{the exact evaluation of $(A)$ requires one round of communication with all retained clients for a given value of $\vv$}. 
Since multiple SGD iterations are necessary to compute the Hessian, this may further exacerbate communication costs. 
To address this concern, we present the following approximation procedure and provide a theoretical quantification of its approximation error.

\noindent
\textbf{Hessian approximation}
We estimate term $(A)$ based on the retained client gradient using the pre-trained data. 
Specifically, we update the weights via
\begin{equation}
\label{eq:forgetting_surrogate}
\vw_{c}^{-} = \vw^*  - \Delta \vw\text{,} 
\end{equation}
where $\Delta \vw = \argmin_{\vv}\tilde{\gF}(\vv)$ and
\begin{equation*}
\begin{split}
\tilde{\gF}(\vv) = &~\frac{1}{2n_{p}}\sum_{\vx\in \gD_p} \|\nabla_{w}\vf(\vx;\vw_*^p)\vv\|_2^2 + \frac{\mu}{2}\|\vv\|_2^2 \\
&~- \nabla^{\top} \gL_{\text{MSE}}(\vw^*;\gD^{-} )\vv.
\end{split}
\end{equation*}

We further use the computational trick in~\cite{mu2020gradients} to efficiently compute the Jacobian-Vector product $\nabla_{w}\vf(\vx;\vw_*^p)\vv$ with a single forward pass.
The proposed approximation offers several advantages, originating mainly from the fact that the objective $\tilde{\gF}(\vv)$ now relies solely on $D_{p}$. 
Consequently, the unlearning process described in~\eqref{eq:forgetting_surrogate} only requires the server and does not need the involvement of any clients in the system.
This characteristic aligns with real-world applications, which improves the feasibility and practicality of our approach.
Moreover, the \remove{} can also be applied for FMs with linear probing without any additional cost, similar to the reasons for \train{} with FMs.
The proposed FU algorithm, termed \remove{}, is given in Algorithm~\ref{alg:forgetting}.

\begin{algorithm}[htp]
\begin{algorithmic}
\STATE \textbf{Input}: output $\hat\vw$ of Alg.~\ref{alg:learning}, learning rate $\eta_{r}$, initial $\vv$, pre-train dataset $\gD_{p}$
\STATE Local machine send gradient to compute $\nabla \gL_{\text{MSE}}(\hat\vw; \gD^{-})$
\FOR {batch $b$ with size $B$ in $1,\ldots, T$}
\STATE{Let $\hat g$ be the gradient of $\tilde \gF$ based on batch $b$ at $\vv$}
\STATE{$\vv \leftarrow \vv - (\eta_{r}/B)\hat g$}
\ENDFOR

\STATE On server, update $\vw^{-} \leftarrow \hat\vw - \vv$
\RETURN{$\vw^{-}$}
\caption{Proposed \remove{}}
\label{alg:forgetting}  
\end{algorithmic}
\end{algorithm}

\subsection{Theoretical Properties}

In this section, we quantify the theoretical difference between the model weight of our proposed method and the one based on retraining from scratch.

Recall that we have $C$ clients with $\gD_{c}$ being the dataset on the $c$-th client.
We first train a linear model $\tilde \vf(\vx;\vw)$ using our proposed method with all datasets $\{\gD_{c}\}_{c=1}^{C}$ and let $\vw^*$ be as defined in~\eqref{eq:global_optimal}.
We consider the scenario in which the client $c$ decides to withdraw from the FL system and asks to remove the information from its dataset from the learned model weight $\hat\vw$ using \train{}.
Consider two different approaches:
\label{method:theo}
\begin{enumerate}
    \item \textbf{Retrain}: Let $\hat\vw^{r}_{c}$ be the weights obtained by retraining from scratch using our proposed FedAvg without the $c$-th client's dataset.

    \item \textbf{Removal}: The proposed unlearning procedure in~\eqref{eq:forgetting_surrogate} and let 
    $\hat \vw^{-}_{c} = \hat\vw  -\Delta \vw\text{\revision{.}}$
    where $\Delta \vw$ is the SGD output that minimizes $\tilde{\gF}(\vv)$.
\end{enumerate}
The difference between $\hat \vw_{c}^{-}$ and $\hat\vw_c^{r}$ has the following upper bound.

\begin{theorem}[Removal and retrain model weight difference (informal)] Let the notation be the same as Lemma~\ref{lemma:fedavg_progress}.
Let $B$ be the SGD batch size in~\eqref{eq:forgetting_surrogate}.
Let $\tilde F_i(\vw) =  \|\nabla_{w}\vf(\vx_{p}^{i};\vw_*^p)\vw\|^2/2 + \mu\|\vw\|^2/2 -\nabla^{\top}\gL_{\text{MSE}}(\hat{\vw};\gD^-)\vw$, $w_i^* = \argmin \tilde{F}_i(\vw)$, and $\xi_{\tilde F} = n_p^{-1} \sum_{i=1}^{n_p} \{\tilde F_i(\vw^*) - \tilde F_i(\vw_i^*)\}$.
As $T\to\infty$, we have
\begin{align*}
\sE\|\hat \vw_{c}^{-} - \hat\vw^{r}_{c}\|^2\leq \underbrace{4\Delta}_{(1)}+ \underbrace{\frac{4\xi_{\tilde F}}{B\beta}}_{(2)} + \underbrace{\frac{4\|\Delta G\|^2}{\lambda_{\min}^2(\nabla^2\gL(\gD_{p}))}}_{(3)}\text{\revision{.}}
\end{align*}
where $G(\gD) = |\gD|^{-1}\sum_{\vx\in \gD} \nabla_{w} \vf(\vx; \vw_*^{p})\nabla_{w}^{\top} \vf(\vx; \vw_*^{p})$ is the gram matrix of an NTK kernel based on dataset $\gD$, $\Delta G = G(\gD^{-}) -G(\gD_{p})$ is the difference between two-gram matrices, and $\lambda_{\min}$ is the smallest eigenvalue.
\end{theorem}
The proof is deferred to supplementary material C5.
To see the bound at a high level, let $\vw_{-}^*$ be the optimal weight based on the remaining dataset and $\Delta \vw^* = \vw_{-}^* - \vw^*$,
the difference can be decomposed into
\begin{align*}
&\sE\|\hat \vw_{c}^{-} - \hat\vw^{r}_{c}\|^2 \\
\leq & 2 \underbrace{\sE\|\hat \vw_{c}^{r} - \vw_{-}^*\|^2 + 2\sE\|\hat \vw - \vw^*\|}_{T_1} + 
\underbrace{2\sE\|\Delta \vw - \Delta \vw^*\|^2}_{T_2}\text{\revision{.}}
\end{align*}
where $T_1$ corresponds to the error from the FedAvg algorithm and is bounded by $(1)$, $T_2$ corresponds to two terms: a) the error from SGD algorithm to compute the $\Delta \vw$ and is bounded by $(2)$, 
b) the error term induced by Hessian approximation using the dataset from the server and is bounded by $(3)$.

With the IID assumption of the data, $\Delta G$ is small based on the law of large numbers.
A more representative sample from the server gives a non-singular gram matrix and larger smallest eigenvalue, hence leading to better reduction.

\subsection{Analysis}
In this section, we discuss the communication overhead and storage costs associated with \ours.

\begin{itemize}[leftmargin=*]
    \item \textbf{Communication overhead.} 
    During unlearning, the \remove{} process is applied solely on the server side, following~\eqref{eq:forgetting_surrogate}. It only requires the gradients from the last epoch of all clients, which are uploaded to the server as part of the standard FedAvg algorithm. Unlike other FU algorithms, \ours{} does not involve additional retraining or fine-tuning, resulting in zero communication overhead.

    \item \textbf{Historical model storage.} In \train{}, no historical models are stored at any point, nor are they used by \remove{}, resulting in zero model storage costs.
\end{itemize}
\section{Experiments}
\label{exp}
In this section, we present experiments to demonstrate the performance of \ours{} across different datasets, model architectures, and various settings. Sec.~\ref{exp:metrics} introduces the FU metrics, followed by Sec.~\ref{exp:data}, which describes the \textit{six} datasets used. Sec.~\ref{exp:fl_scenario} outlines the \textit{two} FL setups, including FL with FM. Sec.~\ref{exp:server_rate} details the experiment reproducibility. In Sec.~\ref{exp:baseline}, we compare \ours{} with baseline methods. Sec.~\ref{exp:regularization} investigates a key hyperparameter, regularization strength, in \ours{}. Finally, Sec.~\ref{exp:clients} verifies the stability of \ours{} across different numbers of clients and varying target client ratios.

\subsection{FU metrics}
\label{exp:metrics}
We follow the metrics commonly used in existing FU studies to use \textit{backdoor attacks (BA)} to validate the effectiveness of unlearning.\footnote{An alternative metric is \textit{membership inference attacks (MIA)}, which attempts to identify whether a given data point was part of the training set by analyzing model properties such as logits or gradients. 
However, the success of MIA is highly dependent on the tendency of the model to overfit the training set~\citep{shokri2017membership}.
We have conducted 27 MIA experiments across different datasets, MIA strategies, and numbers of clients. We find that our model from \train{} generalizes well on the test set, hence MIA is not an appropriate metric.}
BA is a security threat where triggers (such as patches) are inserted into training data, known as \textit{poisoned data}. 
This causes the model to misclassify inputs that contain the trigger while maintaining normal performance on clean data. 
Unlearning is used after BA to remove poisoned data information from the learned model. The effectiveness of a BA is evaluated using two metrics: \textit{Backdoor Success Rate (BSR)} and \textit{Test Accuracy (TA)}.

\begin{itemize}[leftmargin=*]
    \item \textbf{BSR} measures the proportion of data where the model is manipulated to misclassify poisoned data to the attacker's target label. A higher BSR indicates a more successful BA. In FU, we expect a high BSR after the FL model is trained with poison data from malicious client, and a BSR close to random guessing after unlearning this client.
    \item \textbf{TA} measures the accuracy of server model on clean test data. A higher TA indicates better performance on un-poisoned test data. In FU, we expect the TA of our method after removing the impact of the malicious client to be close to that achieved by retraining the FL model from scratch without the malicious client.
\end{itemize}

The optimal FU outcome would significantly reduce the BSR to the level of random guessing, demonstrating effective mitigation of the backdoor, while keeping TA close to the model's performance prior to the attack.

\subsection{Datasets}
\label{exp:data}
To evaluate the effectiveness of \ours{}, we extensively perform experiments on \textbf{six} datasets, including the \texttt{MNIST}~\citep{lecun1998gradient}, \texttt{FashionMNIST}~\citep{xiao2017fashion}, \texttt{CIFAR10}~\citep{krizhevsky2009learning}, \texttt{ImageNet}~\citep{deng2009imagenet}, \texttt{DomainNet}~\citep{peng2019moment}, and \texttt{Flowers} datasets. 
Note that most of the existing FU studies only adopt \texttt{MNIST}, \texttt{FashionMNIST}, and \texttt{CIFAR10}~\citep{liu2021federaser,halimi2022federated}, whose tasks are relatively simpler. 
We also include more complex and larger scale datasets, \texttt{ImageNet}, \texttt{DomainNet}, and \texttt{Flowers} in our study. 
These six datasets are introduced in supplementary material~\ref{app:data}.

\subsection{Different FL Scenarios}
\label{exp:fl_scenario}
Our paper studies two FL scenarios: FL for full parameter updating and FL for the foundation model with linear probing (FM-LP).

\begin{itemize}[leftmargin=*]
\item \textbf{FL for full parameter updating}. For \texttt{MNIST} and \texttt{FashionMNIST} datasets, we use a three-layer fully connected architecture with hidden layer dimensions of $500$ and $100$, which is detailed in supplementary material~\ref{app:model}. 
We randomly sample 10\% images as the holdout public data to pre-train the neural networks. 
For \texttt{CIFAR10} dataset, we adjust the architecture from PyTorch's official classification tutorial and pre-train it on ImageNet~\footnote{https://pytorch.org/tutorials/beginner/blitz/cifar10\_tutorial.html}. 

\item \textbf{FL for FM-LP}. We also explore the use of our methods for larger datasets and more complex models, such as FMs. FMs are pretrained on vast and diverse datasets and the pre-trained FM encoders are used to extract the embedding of the data. Our paper attempts to use this embedding to train a linear head to perform classification~\citep{tsouvalas2023federated}. 
We include CLIP~\citep{radford2021learning} and BLIP2~\citep{li2023blip} as FM backbones, as they are among the most popular FMs in the vision-language domain. We adapt these models to explore the potential of \ours{} within their vision encoders. CLIP is applied to \texttt{ImageNet} and \texttt{DomainNet}, while BLIP2 is used on \texttt{Flowers}. 
\end{itemize}

\subsection{Experiment Details} 
\label{exp:server_rate}
In \train{}, FedAvg\citep{mcmahan2017communication} is used as the basic FL aggregation strategy like existing FU studies. We simulate the BA by randomly selecting one client and poisoning all data on this client as follows. 
For the \texttt{MNIST} dataset, a white patch of size $5 \times 5$ is used as the trigger. 
For \texttt{FashionMNIST}, the trigger is a white patch of size $7 \times 7$. 
In the case of \texttt{CIFAR10}, we adapt the trigger from~\cite{saha2020hidden} and resize it to $10 \times 10$. 
The trigger is positioned at the bottom-right corner of the image. 
All poisoned images (the image with the trigger) have their label flipped to first class. 
For training, the Adam optimizer is used with a learning rate of $10^{-4}$ for the three datasets. 
\texttt{MNIST} is optimized for $400$ epochs, while \texttt{FashionMNIST} and \texttt{CIFAR10} are optimized for $800$ epochs to fully inject the backdoor data effect. 
After \train{}, all clients upload their last epoch gradient to the server.
We then unlearn the information from the poisoned client (i.e., target client) after \train{} for demonstration. 

For baselines in Sec.~\ref{exp:baseline}, we use the same local update steps and global communication rounds for all FL methods and individually tune key hyperparameters, such as the learning rate, for each method. 
Subsequently, we present the impact of key parameters on \ours{} in Sec.~\ref{exp:regularization}. 

\subsection{Comparison to Baselines}
\label{exp:baseline}
We perform \textit{FL for full parameter updating} experiments on \texttt{MNIST}, \texttt{FashionMNIST}, and \texttt{CIFAR10}, the commonly used datasets in existing FU studies~\citep{liu2021federaser,halimi2022federated,nguyen2024empirical}. 
In addition, we perform \textit{FL for FM-LP} on \texttt{ImageNet} (sampled), \texttt{DomainNet}, and \texttt{Flowers}, which is the very first setup scenario in our paper. 
All experiments are compared with existing baselines.
The raw data of the figures are presented in Table~\ref{tab:raw-mnist}-Table~\ref{tab:raw-flowers} in supplementary material.

\paragraph{FL for full parameter updating} Fig.~\ref{fig:baseline} compares the existing popular and publicly reproducible FU strategies, where they are introduced in supplementary material section~\ref{related} (FedEraser, EWSGA, SGA, Flipping, and CF). 
Fig.~\ref{fig:baseline} (a)-(c) visualizes the results on \texttt{MNIST}, \texttt{FashionMNIST}, and \texttt{CIFAR10} respectively. 
The red and blue bars represent TA and BSR for each method while the horizontal line represents the performance of the re-trained model. 
The purple hatched bars represent the gap between each baseline and the retrained model.
\textbf{The shorter the purple hatched bar, the better the performance of the method.}
\begin{figure*}[htbp]
\centering
\includegraphics[width=0.85\textwidth]{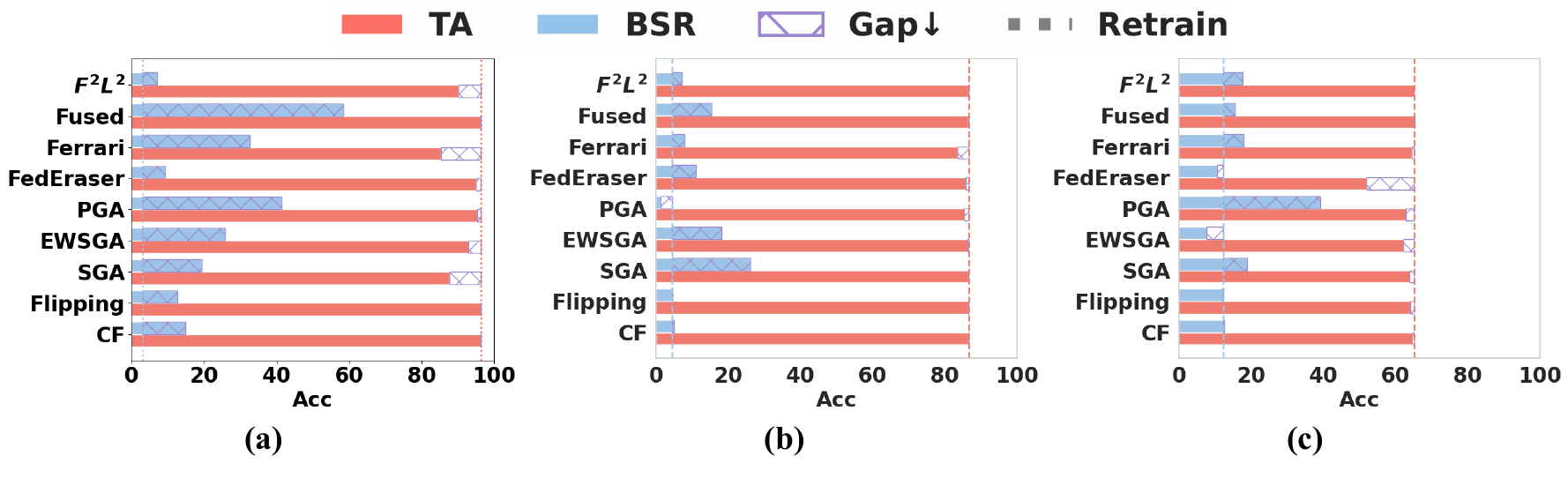}
\captionsetup{justification=raggedright}
\caption{\revision{\textbf{Comparison to baseline FU strategies listed in Tab.~\ref{tab:baseline} for (a) \texttt{MNIST}; (b) \texttt{Fashion-MNIST} and (c) \texttt{CIFAR-10}.} 
The red and blue dashed lines indicate the TA and BSR of the retrained model, respectively. The red and blue bars represent the TA and BSR of each method, while the purple hatched bars highlight the gap between each method and retraining from scratch.\protect\footnote{We only include popular and publicly reproducible FU strategies, which align with recent empirical benchmark FU studies~\cite{nguyen2024empirical}.}}}
\label{fig:baseline}
\end{figure*}
\begin{figure*}[t]
\centering
\includegraphics[width=0.85\textwidth]{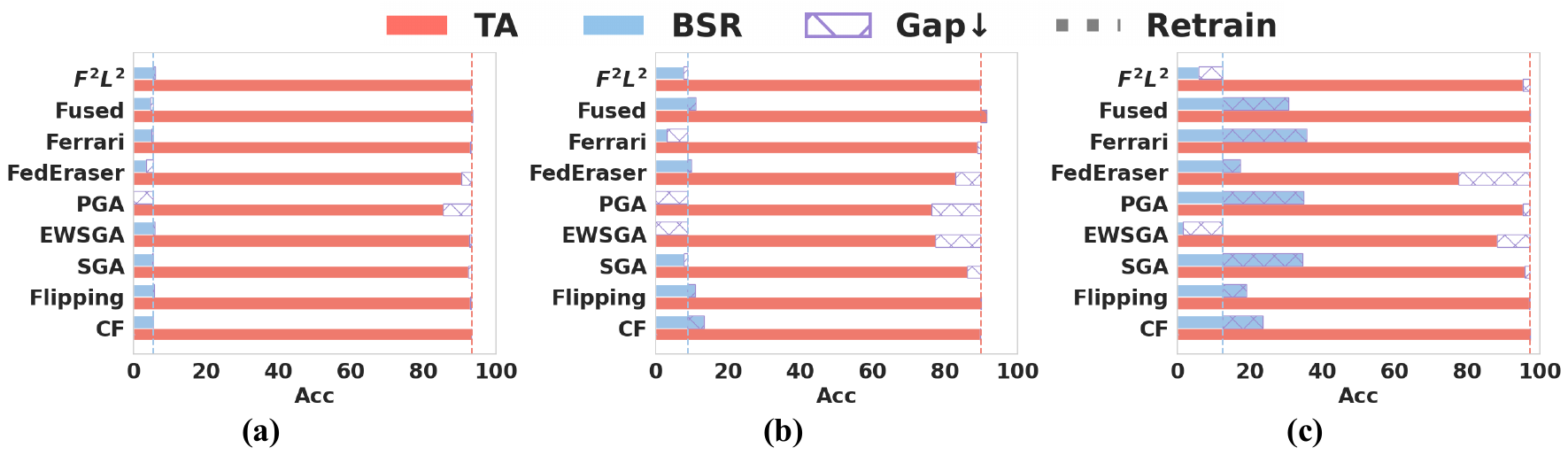}
\captionsetup{justification=raggedright}
\caption{\textbf{Baseline FM comparisions for (a) \texttt{ImageNet} (sampled); (b) \texttt{DomainNet} and (c) \texttt{Flowers}.} 
The red and blue dashed lines indicate the TA and BSR of the retrained model, respectively. The red and blue bars represent the TA and BSR of each method, while the purple-hatched bars highlight the gap between each baseline.}
\label{fig:baseline_fm}
\end{figure*}
As shown in Fig.~\ref{fig:baseline}, across all three datasets and scenarios, \ours{} consistently shows the small gaps in both TA and BSR, nearly matching the retrained model's performance. 
This indicates superior unlearning performance in removing BA while minimizing the impact on TA. 
EWSGA and SGA demonstrate moderate performance, while Flipping and CF exhibit the largest gaps, making them less effective at unlearning. Although EWSGA and SGA are reasonably effective, they do not achieve the optimal performance of \ours{}. FUSED and Ferrari show strong unlearning performance on \texttt{FashionMNIST} and \texttt{CIFAR-10}, but are less consistent on \texttt{MNIST} across multiple runs. Flipping and CF, while exhibiting a smaller performance gap, incur substantial computational costs comparable to retraining.

It is also worth noting that existing FU methods, such as FedEraser, often require substantial memory to store historical models, while methods like EWSGA demand additional communication during the unlearning process, making them inefficient as model sizes grow. 
\textit{In contrast, \ours{} neither requires extra memory nor additional communication during unlearning to achieve the ground truth performance of retraining.}
In summary, \ours{} outperforms existing methods by maintaining a high TA while significantly reducing the BSR after unlearning.

\paragraph{FL for FM-LP} Fig.~\ref{fig:baseline_fm} compares the same baselines under the FM-LP implementation. Fig.~\ref{fig:baseline_fm} (a)-(c) visualizes the results on \texttt{ImageNet} (sampled), \texttt{DomainNet}, and \texttt{Flowers}, respectively. As shown, leveraging embeddings from FMs enables TA to reach as high as 95\% on \texttt{Flowers}, and around 90\% on \texttt{ImageNet} and \texttt{DomainNet}.
The BA is also very high after the backdoor attack, making them suitable for validating FU, with BSRs reaching up to $94\%$ for \texttt{Flowers} and around $70\%-80\%$ for \texttt{ImageNet} and \texttt{DomainNet}. In contrast, the gold-standard \textit{retrained} model achieves a test accuracy similar to the pre-FU models but reduces the BSR to as low as $5\%-10\%$, approximating random guessing.

We then apply these baseline FU strategies to unlearn the backdoored client. \ours{} consistently shows the small gaps among all groups: the TA and BSR of \ours{} remain very close to the \textit{retrained} model, with a TA around $90\%$ and the BSR reduced to the level of random guessing. In contrast, SGA-based methods yield the largest gaps in BSR on \texttt{ImageNet} (sampled) and \texttt{Flowers}. FedEraser, Flipping, and CF show larger BSR gaps compared to the \textit{retrained} group after FU. Moreover, \ours{} operates entirely on the server and does not require additional communication, unlike these baselines. These results suggest that \ours{} demonstrates strong performance in unlearning in the era of FMs.

\subsection{Effect of Regularization Strength}
\label{exp:regularization}
\revision{The regularization term in Eq.~\eqref{eq:forgetting_surrogate} plays a pivotal role in ensuring that the optimal weight in Eq.~\eqref{eq:global_optimal} is well-defined.
Its magnitude, however, requires careful tuning: if too large, the regularization term can dominate the loss function and hinder effective unlearning, while if too small, the Hessian becomes ill-defined, potentially degrading performance. 
To empirically study this trade-off, we conduct experiments to examine how the regularization strength influences the performance of our proposed \ours{} method. 
The results, shown in Fig.~\ref{fig:reg}, reveal notable differences across datasets.}


\begin{figure}[htbp]
\centering
\includegraphics[width=\linewidth]{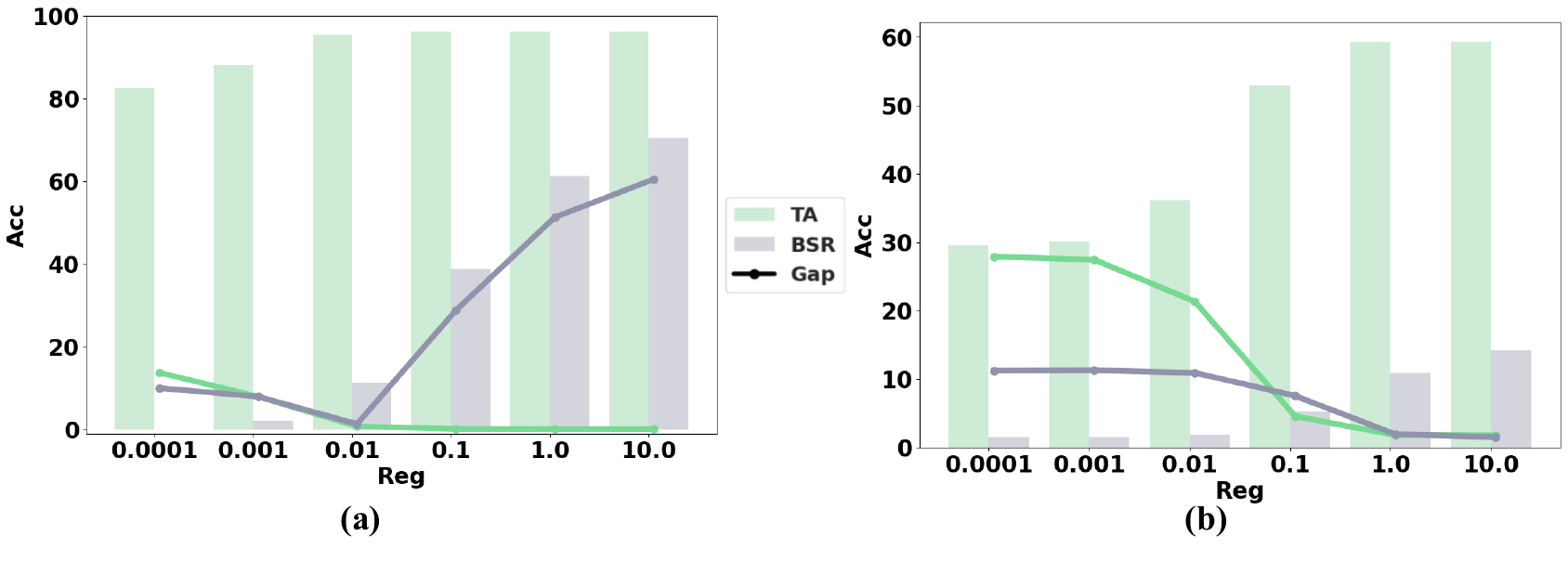}
\captionsetup{justification=raggedright}
\caption{\textbf{The effect of regularization strength for (a) \texttt{MNIST} and (b) \texttt{CIFAR-10} during \remove.} The curve visualizes the \textit{gap} between \remove{} and \textit{retrained} model's TA/BSR.}
\label{fig:reg}
\end{figure}
\footnotetext{The TA before unlearning and after retraining for MNIST is $96.13\%$ and $96.08\%$, which visually overlaps both lines in subfig. (a).}

\revision{For the \texttt{MNIST} dataset (Figure~\ref{fig:reg}a), increasing the regularization strength ($\mu$) from $1 \times 10^{-4}$ to $10$ leaves the TA near-perfect, while the gap between TA and the unlearning measure BSR is minimized at $\mu = 0.01$, indicating optimal performance. 
Beyond this point, BSR sharply increases, widening the gap and suggesting that an excessively large $\mu$ causes the regularization term to dominate, thereby impairing unlearning.

In contrast, the \texttt{CIFAR-10} dataset (Figure~\ref{fig:reg}b) exhibits a different pattern. 
Here, \ours{} remains close to the retrained model for larger values of $\mu$ (e.g., $\mu \geq 1$), showing robustness when the regularization is strong. 
However, the method is more sensitive to smaller $\mu$, indicating that for this more complex dataset, a relatively larger regularization is beneficial.

In summary, these findings indicate that the regularization strength $\mu$ must be carefully tuned. 
It must be large enough to maintain the strong convexity of the objective in Eq.~\eqref{eq:forgetting_surrogate} but not so large that it dominates and diminishes the influence of other critical terms, with the optimal balance differing between simpler and more complex datasets.
}

\subsection{Stability over Scalable Number of Clients}
\label{exp:clients}
One challenge in FL is maintaining the stability of the algorithm as the number of clients scales. 
To investigate this, we examine the effectiveness of \ours{} with 5 to 20 clients on the \texttt{MNIST} dataset
\footnote{As the number of clients increases, the effectiveness of the BA diminishes because the gradient contribution of the poisoned client is diluted when averaged across a larger number of clients. 
We choose \texttt{MNIST} for this study because the BA is most successful against it across all client configurations tested. 
This makes the impact of unlearning more easier to observe.}.
By scaling up the number of clients, we enable a comprehensive evaluation of \remove{}'s efficacy under varied client unlearning scenarios, with particular attention to cases involving a small unlearning ratio (the proportion of data to be unlearned).

Overall, \ours{} demonstrates stability and effectiveness across different numbers of clients, maintaining high accuracy while effectively unlearning the backdoor, regardless of client count.
In Fig.~\ref{fig:Clients}, subfigures (a) and (b) present the TA and BSR, respectively, for client counts of 5, 10, 15, and 20. 
\begin{figure}[htbp]
\centering
\includegraphics[width=0.99\linewidth]{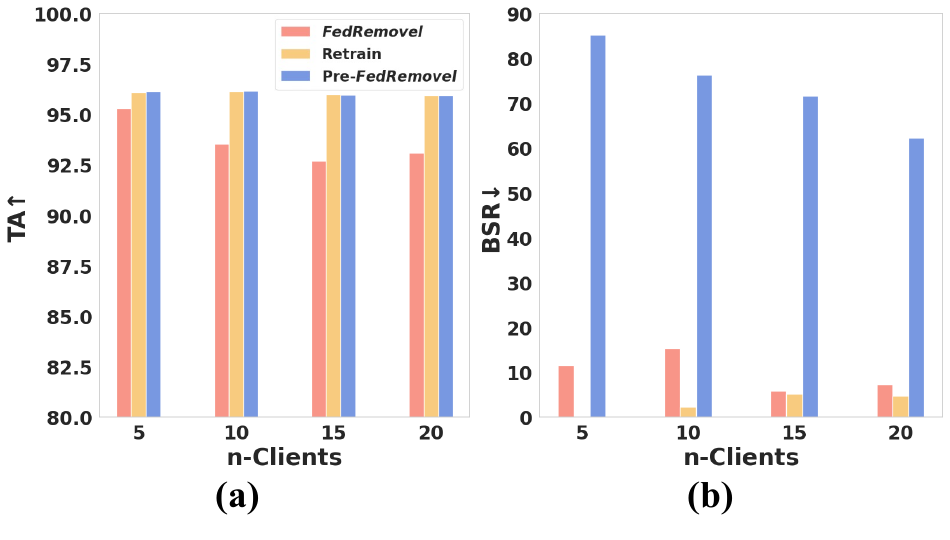}	
\captionsetup{justification=raggedright}
\caption{\textbf{\remove{}'s performance on the different number of clients for \texttt{MNIST}}. (a) presents TA and (b) represents BSR.}
 \label{fig:Clients}
\end{figure}
In each scenario, one client is randomly selected to inject the backdoor data and later undergoes the unlearning process. The blue bar represents the pre-\remove{} scenario, the pink bar shows the results after applying \remove{}, and the orange bar represents the performance of a \textit{retrained} model.
Across all client counts, TA and BSR reach nearly $100\%$ immediately after training, indicating a successful BA on the server model. 
For unlearning, \remove{} maintains a high TA, close to that of the retrained model in all cases. 
Specifically, for 5 clients, \remove{} achieves around $95\%$ TA, while the BSR drops significantly, reflecting effective unlearning. As the number of clients increases to 10, 15, and 20, \remove{} performs consistently well, with almost no decrease in TA and a BSR reduced to the level of random guessing. 
This consistent performance underscores the robustness of \ours{} across various client configurations, making it a reliable method for FU.

\subsection{Computational Efficiency}
\label{sec:efficiency}
\begin{table}[t]
  \centering
  \footnotesize
  \setlength{\tabcolsep}{3pt}
  \caption{\textbf{Per-epoch computational cost (server-side ``w/ Approx.''). standard error.}}
  \label{tab:efficiency}
  \begin{tabular}{@{}lclrrr@{}}
    \toprule
    Dataset & Approx & Op calls & FLOPs\,(1e9) & Thrp.\,(s/s) & Comm. \\
    \midrule
    \multirow{2}{*}{MNIST}        & w/o & 2,325,664 &  30.10 & 1.20 & $O(n)$ \\
                                  & w/  &   322,281 &   4.19 & 0.20 & 0   \\[2pt]
    \multirow{2}{*}{FashionMNIST} & w/o & 2,325,664 &  30.10 & 1.20 & $O(n)$ \\
                                  & w/  &   322,281 &   4.19 & 0.20 & 0   \\[2pt]
    \multirow{2}{*}{CIFAR10}      & w/o & 3,485,004 & 205.10 & 0.10 & $O(n)$ \\
                                  & w/  &   483,549 &  28.51 & 0.08 & 0   \\
    \bottomrule
  \end{tabular}
\end{table}

The Hessian approximation enables the entire \ours{} unlearning
procedure to run \emph{solely on the server}.  As a result  
(i) no messages are exchanged beyond the final-epoch gradients already
required by FedAvg, and  
(ii) no client checkpoints are retained—yielding \textbf{zero extra
communication} and \textbf{zero model-storage overhead}.

Table \ref{tab:efficiency} presents the one epoch of \ours{} \textbf{with}
(\,``w/ Approx.''\,) and \textbf{without} (\,``w/o Approx.''\,) the
approximation on three datasets using the same model as in
Sec.\ref{exp}. As shown, the approximation significantly reduces both computation and communication across all datasets. Since the residual computation is performed entirely on the server, client-side throughput remains unaffected, and communication overhead stays at zero in the ``w/ Approx.'' group. The lower FLOPs and throughput values simply reflect the smaller unlearning workload handled by the server. Overall, the approximation provides substantial efficiency gains while preserving the core strength of \ours{}: fully server-side unlearning with no additional communication or storage cost.

\section{Conclusion}
Motivated by the mechanism of machine unlearning under quadratic loss, this paper presents a novel FU framework, \ours{}, that removes the need for access to the target client and retaining clients, storage of historical models, or additional communication. By leveraging a linear approximation in parameter space through \train{} and introducing \remove{} for efficient client forgetting, \ours{} achieves certified unlearning with theoretical guarantees that bound its deviation from full retraining. We validate its effectiveness across a range of datasets and demonstrate its adaptability to FMs such as CLIP within FL.

This work opens several promising directions for future research, including extending \ours{} to broader FL settings beyond classification, applying it to generative and structured prediction tasks, exploring nonlinear approximations for greater expressiveness, and improving robustness in the presence of complex client heterogeneity. We also envision deeper integration of \ours{} with emerging FMs and scaling its application to more diverse, real-world federated environments.

\bibliographystyle{IEEEtran}
\bibliography{biblio}

\renewcommand{\appendixname}{Supplementary Material}
\clearpage
\newpage
\appendix
\subsection{Notation Table}
A summary of notations used in the main paper is given in Tab.~\ref{tab:notation}.

\label{app:notation}
\begin{table}[ht]
\caption{Important notations}
\resizebox{\linewidth}{!}{%
\centering
\begin{tabular}{cl}
\toprule
Notations & Description \\ 
\midrule
$\gX$ & feature Space \\
$\gY$ & label Space \\
$K$ & number of classes \\
$C$ & set of clients \\
$\vf : \gX \to \sR^{K}$ & classification model \\
$d$ & input dimension \\
$R$ & global epochs \\
$\vw$ & model weight \\
$\vw^p$ & pre-trained model weight \\
$\vw_s^r$ & server model at round $r$ \\
$\theta$ & FM encoder \\
$\phi$ & Linear layer \\
$\gL$ & loss function \\
$\gD$ & full train set \\
$\gD_p$ & pre-trained data \\ 
$\gD_c$ & dataset on $c$-th client \\ 
$\gD_f$ & dataset on forgetten client \\
$\gD^-$ & union of dataset for retain client(s) \\
$\gD_\texttt{test}$ & test set \\
$M$ & local steps \\
$\gR$ & risk objective \\
$\mu$ & regularization coefficient \\
$\eta$ & step size \\
$p_c$ & aggregation weight of client $c$ \\
$\sigma^2$ & upper bound variance of stochastic gradient \\
$g_c^{r,m}$ & gradient for client $c$, global epoch $r$ and local step $m$ \\
$n_p$ & number of pre-trained data \\
\bottomrule
\end{tabular}%
}
\label{tab:notation}
\end{table}

\subsection{Related Work}
\label{related}
\newif\ifcondition
\conditiontrue 

\ifcondition
Federated Unlearning (FU) has emerged as a critical research area in response to growing privacy concerns and regulatory requirements~\citep{parisi2019continual,liu2021federaser,li2021anti,nguyen2022survey,mercuri2022introduction,jin2023forgettable,gogineni2024efficient,nguyen2024empirical}. 
Recent studies have explored various strategies for removing client data from federated models while maintaining model utility~\citep{nguyen2022survey, nguyen2024empirical}. 
An ideal FU method should satisfy two key requirements: (1) minimize dependence on both target and retained clients during the unlearning process, and (2) eliminate the need for storing historical models to reduce storage overhead.

As demonstrated in Table~\ref{tab:baseline}, existing approaches fail to simultaneously meet all these criteria. 
We systematically categorize current FU methods into five paradigms based on their underlying unlearning strategies:
\begin{table}[htbp]
\centering
\caption{\textbf{Comparison between \ours{} (Ours) and existing FU strategies.}\protect \footnotemark}
\label{tab:baseline}
\resizebox{0.95\linewidth}{!}{
\begin{tabular}{lccccc}
\toprule
Baseline & 
\makecell{No extra \\  commu.} & 
\makecell{No target \\ client(s) \\ involve} & 
\makecell{No retain \\ client(s) \\ involve} & 
\makecell{No historical \\ model \\ storage} & 
\makecell{Certified\\ unlearn} \\ 
\midrule
Retrain~\citep{mercuri2022introduction} & \textcolor{red}{\faTimes} & \textcolor{green}{\faCheck} & \textcolor{red}{\faTimes} & \textcolor{green}{\faCheck} & N/A \\ 
\midrule
CF~\citep{parisi2019continual} & \textcolor{red}{\faTimes} & \textcolor{red}{\faTimes} & \textcolor{red}{\faTimes} & \textcolor{green}{\faCheck} & \textcolor{red}{\faTimes} \\ 
\midrule
Flipping~\citep{nguyen2024empirical} & \textcolor{red}{\faTimes} & \textcolor{red}{\faTimes} & \textcolor{green}{\faCheck} & \textcolor{green}{\faCheck} & \textcolor{red}{\faTimes} \\ 
\midrule
SGA~\citep{li2021anti} & \textcolor{red}{\faTimes} & \textcolor{red}{\faTimes} & \textcolor{green}{\faCheck} & \textcolor{green}{\faCheck} & \textcolor{red}{\faTimes} \\ 
\midrule
EW-SGA~\citep{wu2022federated} & \textcolor{red}{\faTimes} & \textcolor{red}{\faTimes} & \textcolor{green}{\faCheck} & \textcolor{green}{\faCheck} & \textcolor{red}{\faTimes} \\ 
\midrule
PGD~\citep{halimi2022federated} & \textcolor{red}{\faTimes} & \textcolor{red}{\faTimes} & \textcolor{red}{\faTimes} & \textcolor{green}{\faCheck} & \textcolor{red}{\faTimes} \\ 
\midrule
KNOT~\citep{su2023asynchronous} & \textcolor{red}{\faTimes} & \textcolor{green}{\faCheck} & \textcolor{red}{\faTimes} & \textcolor{red}{\faTimes} & \textcolor{red}{\faTimes} \\ 
\midrule
FedEraser~\citep{liu2021federaser} & \textcolor{red}{\faTimes} & \textcolor{green}{\faCheck} & \textcolor{red}{\faTimes} & \textcolor{red}{\faTimes} & \textcolor{red}{\faTimes} \\ 
\midrule
Starfish~\citep{liu2024privacy} & \textcolor{red}{\faTimes} & \textcolor{green}{\faCheck} & \textcolor{red}{\faTimes} & \textcolor{red}{\faTimes} & \textcolor{green}{\faCheck} \\ 
\midrule
FUSED~\citep{zhong2025unlearning} & \textcolor{red}{\faTimes} & \textcolor{green}{\faCheck} & \textcolor{red}{\faTimes} & \textcolor{green}{\faCheck} & \textcolor{red}{\faTimes} \\ 
\midrule
FedRecovery~\citep{zhang2023fedrecovery} & \textcolor{green}{\faCheck} & \textcolor{green}{\faCheck} & \textcolor{green}{\faCheck} & \textcolor{red}{\faTimes} & \textcolor{red}{\faTimes} \\ 
\midrule
Ferrari~\citep{gu2024ferrari} & \textcolor{red}{\faTimes} & \textcolor{red}{\faTimes} & \textcolor{green}{\faCheck} & \textcolor{red}{\faTimes} & \textcolor{green}{\faCheck} \\ 
\midrule
\textbf{\ours{} (Ours)} & \textcolor{green}{\faCheck} & \textcolor{green}{\faCheck} & \textcolor{green}{\faCheck} & \textcolor{green}{\faCheck} & \textcolor{green}{\faCheck} \\ 
\bottomrule
\end{tabular}}
\end{table}

\begin{enumerate}[leftmargin=*]
\item \textbf{Retraining-based approaches}. 
The most straightforward FU method involves complete retraining of the model using only data from retained clients~\citep{mercuri2022introduction}. 
While this approach guarantees complete unlearning by definition, it suffers from prohibitive computational and communication costs due to requiring full participation from all retained clients. 
The resource-intensive nature of this method makes it impractical for large-scale federated learning systems.

\item \textbf{Fine-tuning-based methods}.
These approaches leverage the phenomenon of catastrophic forgetting in deep neural networks, where models tend to lose previously learned information when trained on new tasks or data. 
\cite{parisi2019continual} and~\cite{jia2023model} demonstrated that targeted fine-tuning on retained data can effectively induce forgetting of target data while preserving performance on the retained dataset. 
\revision{Recent advances in this paradigm include FUSED~\citep{zhong2025unlearning}, which employs parameter-efficient adapters to reduce computational overhead, and Ferrari~\citep{gu2024ferrari}, which incorporates Lipschitz-constrained loss functions for selective feature unlearning. 
Despite their advantages, these methods still require coordination among remaining clients, resulting in significant communication overhead.}


\item \textbf{Flipping-based techniques}
This approach induces forgetting by exploiting deep neural networks' tendency to memorize random labelings of training data~\citep{zhang2021understanding}. The process involves three key steps: (1) the target client relabels its local data randomly, (2) performs multiple training iterations on these corrupted labels, and (3) transmits the resulting model to the server for aggregation~\citep{nguyen2024empirical,gogineni2024efficient}. While effective at removing target data influence, this method presents two significant drawbacks. First, it requires additional communication rounds between server and target client, increasing overhead. More critically, it creates security vulnerabilities by depending entirely on the target client--who may submit a maliciously altered model during the unlearning process.

\item \textbf{Gradient ascent-based approaches}.
These methods achieve unlearning by reversing the optimization process. 
While basic stochastic gradient ascent (SGA)~\citep{li2021anti} can remove target data, they require careful hyperparameter tuning (e.g., learning rate, epochs) to avoid suboptimal unlearning. 
Poorly adjusted SGA may degrade model performance beyond the intended level—even below random guessing. 
To mitigate this,~\cite{wu2022federated} integrates elastic weight consolidation (EWC) to automatically stabilize the SGA process, while \cite{halimi2022federated} employs projected gradient descent (PGD) with gradient clipping and normalization for more reliable optimization. 
A key limitation of SGA-based methods is their dependence on the target client’s cooperation, which introduces risks if the client behaves maliciously during unlearning.

\item \textbf{Historical model-based methods}.
\cite{liu2021federaser} introduced FedEraser, which calibrates retained model updates using historical client weights to accelerate unlearning. 
However, this approach demands storing global and local models at every FL epoch, leading to high storage costs. 
Subsequent works, such as \cite{zhang2023fedrecovery}, extend FedEraser by incorporating clustering for asynchronous unlearning scenarios. 
FedEraser’s reliance on full historical model storage becomes increasingly impractical as model sizes grow, and its calibration process also introduces additional communication overhead.
\end{enumerate}

To summarize, existing methods either require the cooperation of target or retained clients during the unlearning process, or additional storage for historical models, which may be impractical in real-world scenarios, as discussed in Sec.~\ref{sec:intro}. 
Furthermore, these methods may not ensure that the model after unlearning performs comparably to one retrained from scratch. 
To address these limitations, we introduce \ours{}, a \textit{certified} FU algorithm that operates \textit{without client involvement, extra communication, or additional model storage}.

\footnotetext{We only include popular and publicly reproducible FU strategies, which aligns with recent empirical benchmark FU studies~\cite{nguyen2024empirical}. For FedEraser, its calibrated training version involves fine-tuning the calibrated client.}

\else
FU is an emerging field~\citep{parisi2019continual,liu2021federaser,li2021anti,nguyen2022survey,mercuri2022introduction,jin2023forgettable,gogineni2024efficient,nguyen2024empirical}. 
In FU, several factors are taken into consideration, including the certification of the algorithm, storage cost, and client involvement, as illustrated . 

Ideally, a FU algorithm should combine theoretical validation with empirical validation, providing the strongest guarantees of effective unlearning. 
From this perspective, it is beneficial for a FU algorithm to minimize reliance on client participation, as this reduces communication overhead during the unlearning process. Additionally, it is crucial that the FU process does not rely heavily on the target client, as such untrusted clients could exploit the unlearning process by executing malicious algorithms (e.g., poisoning attacks) and distributing compromised models to other clients during continued FL rounds. Finally, it is preferable for the unlearning process to avoid storing historical models, as this can significantly increase storage costs.  

We summarize existing FU algorithms in Tab.~\ref{tab:baseline} based on these considerations.

\begin{table*}[htbp]
\centering
\caption{\textbf{Comparison between \ours{} (Ours) and existing FU strategies.}\protect \footnotemark}
\label{tab:baseline}
\begin{tabular}{lccccc}
\toprule
\textbf{Baseline | Category} & 
\makecell{\textbf{No Extra} \\ \textbf{Communication}} & 
\makecell{\textbf{No Target} \\ \textbf{Client(s)} \\ \textbf{Involvement}} & 
\makecell{\textbf{No Retain} \\ \textbf{Client(s)} \\ \textbf{Involvement}} & 
\makecell{\textbf{No Historical} \\ \textbf{Model} \\ \textbf{Storage}} & 
\makecell{\textbf{Certified Unlearning}} \\ 
\midrule
\textbf{Retrain}~\citep{mercuri2022introduction} & \textcolor{red}{\faTimes} & \textcolor{green}{\faCheck} & \textcolor{red}{\faTimes} & \textcolor{green}{\faCheck} & N/A \\ 
\midrule
\textbf{CF}~\citep{parisi2019continual} & \textcolor{red}{\faTimes} & \textcolor{red}{\faTimes} & \textcolor{red}{\faTimes} & \textcolor{green}{\faCheck} & \textcolor{red}{\faTimes} \\ 
\midrule
\textbf{Flipping}~\citep{nguyen2024empirical} & \textcolor{red}{\faTimes} & \textcolor{red}{\faTimes} & \textcolor{green}{\faCheck} & \textcolor{green}{\faCheck} & \textcolor{red}{\faTimes} \\ 
\midrule
\textbf{SGA}~\citep{li2021anti} & \textcolor{red}{\faTimes} & \textcolor{red}{\faTimes} & \textcolor{green}{\faCheck} & \textcolor{green}{\faCheck} & \textcolor{red}{\faTimes} \\ 
\midrule
\textbf{EW-SGA}~\citep{wu2022federated} & \textcolor{red}{\faTimes} & \textcolor{red}{\faTimes} & \textcolor{green}{\faCheck} & \textcolor{green}{\faCheck} & \textcolor{red}{\faTimes} \\ 
\midrule
\textbf{PGD}~\citep{halimi2022federated} & \textcolor{red}{\faTimes} & \textcolor{red}{\faTimes} & \textcolor{red}{\faTimes} & \textcolor{green}{\faCheck} & \textcolor{red}{\faTimes} \\ 
\midrule
\textbf{KNOT}~\citep{su2023asynchronous} & \textcolor{red}{\faTimes} & \textcolor{green}{\faCheck} & \textcolor{red}{\faTimes} & \textcolor{red}{\faTimes} & \textcolor{red}{\faTimes} \\ 
\midrule
\textbf{FedEraser}~\citep{liu2021federaser} & \textcolor{red}{\faTimes} & \textcolor{green}{\faCheck} & \textcolor{red}{\faTimes} & \textcolor{red}{\faTimes} & \textcolor{red}{\faTimes} \\ 
\midrule
\textbf{Starfish}~\citep{liu2024privacy} & \textcolor{red}{\faTimes} & \textcolor{green}{\faCheck} & \textcolor{red}{\faTimes} & \textcolor{red}{\faTimes} & \textcolor{green}{\faCheck} \\ 
\midrule
\textbf{\textcolor{blue}{FUSED}}~\citep{zhong2025unlearning} & \textcolor{red}{\faTimes} & \textcolor{green}{\faCheck} & \textcolor{red}{\faTimes} & \textcolor{green}{\faCheck} & \textcolor{red}{\faTimes} \\ 
\midrule
\textbf{FedRecovery}~\citep{zhang2023fedrecovery} & \textcolor{green}{\faCheck} & \textcolor{green}{\faCheck} & \textcolor{green}{\faCheck} & \textcolor{red}{\faTimes} & \textcolor{red}{\faTimes} \\ 
\midrule
\textbf{\textcolor{blue}{Ferrari}}~\citep{gu2024ferrari} & \textcolor{red}{\faTimes} & \textcolor{red}{\faTimes} & \textcolor{green}{\faCheck} & \textcolor{red}{\faTimes} & \textcolor{green}{\faCheck} \\ 
\midrule
\textbf{\ours{} (Ours)} & \textcolor{green}{\faCheck} & \textcolor{green}{\faCheck} & \textcolor{green}{\faCheck} & \textcolor{green}{\faCheck} & \textcolor{green}{\faCheck} \\ 
\bottomrule
\end{tabular}
\end{table*}

\begin{itemize}
    \item \textbf{Retraining} from scratch after an unlearning request is a straightforward yet inefficient method for removing previously trained information from the model~\citep{mercuri2022introduction}. Although it eliminates the need for participation from the target client, retraining introduces the highest communication cost, as it requires the involvement of the remaining clients, making it the most resource-intensive approach among all FU strategies.

    \item \textbf{Fine-tuning-based unlearning}, also known as \textit{model reconstruction}, fine-tunes the pre-trained model on the retained dataset to obtain the unlearned model~\citep{jia2023model}. This approach leverages the catastrophic forgetting property of DNNs in continual learning to de-emphasize previously learned information~\citep{parisi2019continual}.  It is widely recognized as a straightforward defense mechanism against backdoor attacks~\citep{liu2017neural,li2022backdoor}. Some studies focus on integrating catastrophic forgetting with parameter-efficient fine-tuning. For example, Zhong et al.~\cite{zhong2025unlearning} fine-tune only the adapter to reduce communication costs and enable reversible unlearning in FU. Ferrari~\citep{gu2024ferrari} also fine-tunes the global model, but replaces the standard cross-entropy objective with a Lipschitz feature-sensitivity loss, enabling unlearning of specific features without any other-client participation. Finally, many existing FU algorithms~\citep{halimi2022federated,liu2021federaser} also plug in fine-tuning as calibration training to further enhance the performance of unlearning. Though fine-tuning does not involve the target client, it also requires the coordination among the remaining clients and thus incurs a high communication cost due to the nature of FL.

    \item \textbf{Flipping-based unlearning}~\cite{nguyen2024empirical, gogineni2024efficient} facilitates unlearning by confusing the model with random labels on the target client's dataset, leveraging the observation that DNNs tend to overfit on random labels~\citep{zhang2021understanding}.  Once a removal request is received, the target client trains a local model on the same data input but with randomized labels. After several iterations, the local model is sent to the server for aggregation, aiding the unlearning process. 
    Therefore, this approach incurs additional communication costs between the server and target clients.
    Moreover, relying solely on the target client for unlearning introduces potential security risks, such as the possibility of the client returning a malicious model after poisoning its training process. 

    \item \textbf{Gradient ascent-based unlearning} treats unlearning as the reverse learning process, achieved through stochastic gradient ascent (SGA). SGA is also a popular defense strategy for backdoor attacks~\citep{li2021anti}. However, the effectiveness of SGA requires meticulous adjustment of hyperparameters such as learning rate and number of epochs and has non-optimal optimization if those are not carefully tuned. Insufficient adjustment may result in the model's performance on the target set exceeding the intended level, such as performing worse than random guessing. To address this,~\cite{wu2022federated} proposes incorporating elastic weight consolidation (EWC) in the SGA process to automatically balance the SGA process. In addition,~\cite{halimi2022federated} introduces projected gradient descent (PGD), which takes advantage of gradient clipping and normalization in gradient ascent to provide better optimization. The SGA-based strategies also rely on the cooperation of the target client, which introduces new risks if the client engages in covert malicious behavior during the unlearning process.

    \item \textbf{FedEraser-based unlearning.}~\citep{liu2021federaser} proposes a calibration method that utilizes the historical updates of clients' parameters to adjust the retained updates, thereby accelerating the unlearning process. However, this approach requires storing both global and local models at every epoch in FL, which increases storage costs. Many FU algorithms are based on FedEraser, such as~\citep{zhang2023fedrecovery}, which further incorporates clustering to handle asynchronous scenarios in FU. FedEraser’s reliance on historical model storage for every global epoch makes it especially impractical if the model sizes are growing larger and larger nowadays, and it incurs additional communication overhead due to the need for calibrated training.
\end{itemize}

To summarize, existing methods either require the cooperation of target or retained clients during the unlearning process, or additional storage for historical models, which may be impractical in real-world scenarios, as discussed in Sec.~\ref{sec:intro}. Furthermore, these methods may not ensure that the model after unlearning performs comparably to one retrained from scratch. To address these limitations, we introduce \ours{}, a \textit{certified} FU algorithm that operates \textit{without client involvement, extra communication, or additional model storage}.

\footnotetext{We only include popular and publicly reproducible FU strategies, which aligns with recent empirical benchmark FU studies~\cite{nguyen2024empirical}. For FedEraser, its calibrated training version involves fine-tuning the calibrated client.}

\fi

\subsection{Weight Difference Bound}
\subsubsection{Notation and Definition}
We consider a $K$-class classification task with feature space $\gX\subset \sR^{d}$ and label space $\gY=[K]$ where $[K] = \{1,\ldots, K\}$.
Under the FL setting, the training dataset is not available on a single device and is partitioned over $C$ clients.
Let $\gD_{c} = \{(x_c^i, y_c^i)\}_{i=1}^{n_c}$ be the training set on the $c$-th client. 
We also have the pre-trained dataset $\gD_{p} = \{(x_p^i, y_p^i)\}_{i=1}^{n_p}$, which can consist of data from public domain or data provided by any of the FL participant.
These datasets are samples from the distribution with density function $p(\vx;y)$.
We denote by $\gD= \{\gD_{1},\ldots,\gD_{C}\}$ the full training set and $n=n_1+\ldots+n_C$.

Let $\vw_*^p$ be the model weight from pretraining a $K$-class classifier on $\gD_{p}$.
Let $\tilde{\vf}(\vx; \vw)=\vf(\vx; \vw_*^{p}) + \langle \nabla_{w} \vf(\vx; \vw_*^{p}),\vw-\vw_*^{p}\rangle$ be the linear approximation of $\vf(\vx;\vw)$.
We train $\tilde{\vf}$ by FedAvg on $\gD$.
The local loss 
\begin{equation*}
    \gL_{\text{MSE}}(\vw;\gD_{c}) = \frac{1}{2n_c} \sum_{i=1}^{n_c} \|\tilde\vf(\vx_c^i;\vw)-\vy_c^i\|^2 + \frac{\mu}{2}\|\vw\|^2.
\end{equation*}
is denoted as $\gL_{c}(\vw)$ for simplicity in the rest of the description and the global loss is
\[\gL(\vw;\gD) = \sum_{c=1}^{C} p_c\gL_{c}(\vw)\]
where $p_{c} = n_c/n$.
We denote
\[\vw_{c}^* = \argmin_{\vw} \gL_{c}(\vw)\]
\[\vw^* = \argmin_{\vw} \gL(\vw)\]

We first outline the FedAvg method in Algorithm~\ref{alg:learning}. 
In round $r$, we perform the following updates: 
\begin{itemize}
    \item Let $S_{c}^{r,m}$ be an index set that is sampled from $[n_c]$ without replacement for the $m$-th batch in $r$-th round on client $c$, and let the stochastic gradients based on mini-batch $S_{c}^{r,m}$
    \begin{equation}
    \label{eq:sgd_grad}
        g_{c}^{r,m}(\vw) = \frac{1}{|S_{c}^{r,m}|} \sum_{i\in S_{c}^{r,m}} \nabla \ell(\vf(\vx_c^i; \vw); y_c^i). 
    \end{equation}

    \item Starting from the shared global parameters $\vw_{c}^{r,0} = \vw_{s}^{r-1}$, we update the local parameters for $m\in[M]$
    \[\vw_{c}^{r,m} = \vw_{c}^{r,m-1} - \eta_{l} g_{c}^{r,m}(\vw_{c}^{r,m-1}).\]
    
    \item Compute the new global parameters via
    \[\vw_{s}^{r} = \sum_{c=1}^{C}p_{c}\vw_{c}^{r,M}\]
\end{itemize}

\begin{definition}[$\mu$-strongly convex]
A function $h(\cdot):\gH \subset \sR^{p} \to \sR$ is called $\mu$-strongly convex if $h(\vw) - \frac{\mu}{2}\|\vw\|^2$ is convex.
\end{definition}

\begin{definition}[$\beta$-smooth]
We say a continuously differentiable function $h(\cdot):\gH \subset \sR^{p} \to \sR$ is $\beta$-smooth if its gradient $\nabla h$ is
$\beta$-Lipschitz, that is
\[\|\nabla h(\vw) -\nabla h(\vv)\|\leq \beta \|\vw-\vv\|
\] 
for and $\vw, \vv \in \gH$.    
\end{definition}

\subsubsection{Conditions and assumptions}

\begin{condition}[Smoothness \& convexity]
The objective function $\gL_{c}(\vw)$ is $\mu$-strongly convex and $\beta$-smooth where 
$\beta = \max_{c\in[C]}\lambda_{\mathrm{max}}(\nabla^2 \gL_{c}(\vw))$ is the largest eigenvalue of all Hessian matrix over $C$ clients.
\end{condition}
Since our local objective function $\gL_{c}$ is the sum of two quadratic functions, the objective function is clearly $\mu$-strongly convex.
Sum of two quadratic functions is also $\beta$-smooth with 
\[\beta = \max_{c\in[C]}\lambda_{\mathrm{max}}(\nabla_{w} \vf(\gD_{c};\vw_*^{p})\nabla_{w}^{\top} \vf(\gD_{c};\vw_*^{p}) + \mu \mI)\]
being the largest eigenvalue of the Hessian of the local objective function.
Having a larger value of $\beta$ also guarantees the $\beta$-smoothness. 
We therefore, take the largest eigenvalue across all $C$ clients and use it as the constant for $\beta$-smoothness.

\begin{assumption}[Unbiased gradients with bounded variances]
\label{assump:bounded_variance}
Let $g_{c}^{r,m}(\vw)$ be as defined in~\eqref{eq:sgd_grad}, then $g_{c}^{r,m}(\vw)$ is an unbiased estimate for the local gradient $\nabla \gL_{c}(\vw)$.
Let $\vw^0$ be the initial value of the proposed FL algorithm for the training of the linear model $\tilde\vf$. 
Let $\gR:=\{\vw:\vw\in \sR^{p},\|\vw\|\leq \|\vw^0-\vw^*\|<\infty\}$ be a bounded region of weight parameter space.
We assume the unbiased stochastic gradients $g_{c}^{r,m}(\vw)$ have a bounded variance. That is $\exists~\sigma >0$ s.t.
\begin{equation}
\label{eq:bounded_grad_variance}
\sE_{S_{c}^{r,m}}\|g_{c}^{r,m}(\vw) - \nabla \gL_{c}(\vw)\|^2\leq \sigma^2    
\end{equation}
for any $\vw\in \gR$.
\end{assumption}

\subsubsection{Some technical lemmas}
\begin{lemma}[Relaxed triangle inequality]
\label{lemma:relaxed_triangle_inequality}
Let $\vw_1,\ldots, \vw_m$ be $m$ vectors in $\sR^{p}$. Then the following are true
\begin{enumerate}
    \item $\|\vw_i + \vw_j\|^2\leq (1+\alpha)\|\vw_i\|^2 + (1+\frac{1}{\alpha})\|\vw_j\|^2$ for any $\alpha>0$
    \item $\|\sum_{i=1}^{m}\vw_i\|^2\leq m\sum_{i=1}^{m}\|\vw_i\|^2$
\end{enumerate}
\end{lemma}
The proof of the Lemma can be found in~\citet[Lemma 3]{karimireddy2020scaffold}.

\begin{lemma}[Implication of $\mu$-strongly convexity and $\beta$ smoothness]
If $h:\gH \subset \sR^{p} \to \sR$ is continuously differential and $\mu$-strongly convex, then
\[h(\vw)\geq h(\vv) + \langle \nabla h(\vv), \vw-\vv\rangle + \frac{\mu}{2}\|\vw-\vv\|^2\]
for any $\vw,\vv\in \gH$.

If $h$ is also $\beta$-smooth, then
\[h(\vw)\leq h(\vv) + \langle \nabla h(\vv), \vw-\vv\rangle + \frac{\beta}{2}\|\vw-\vv\|^2\]
for any $\vw,\vv\in \gH$.
Let $\vw^* = \argmin_{\vw\in \gH} h(\vw)$, we also have
\begin{equation}
\label{eq:difference_upper_bound}
    h(\vw)-h(\vw^*)\leq \frac{\beta}{2}\|\vw-\vw^*\|^2
\end{equation}
and 
\begin{equation}
\label{eq:grad_bound}
    \|\nabla h(\vw)\|^2\leq 2\beta\{h(\vw)-h(\vw^*)\}
\end{equation}

\end{lemma}

\begin{lemma}[Perturbed strong convexity]
\label{lemma:perturbed_strong_convexity}
Let $h(\cdot):\gH \subset \sR^{p} \to \sR$ be a $\beta$-smooth and $\mu$-strongly convex function, then for any $\vu, \vv,\vw\in \gH$, we have
\[
\langle \nabla h(\vu), \vw-\vv\rangle \geq h(\vw) - h(\vv) + \frac{\mu}{4}\|\vv-\vw\|^2-\beta\|\vw-\vu\|^2.
\]
\end{lemma}
The proof of the Lemma can be found in~\citet[Lemma 5]{karimireddy2020scaffold}.

\begin{lemma}
\label{lemma:SGD_progress}
Let $h(\vw) = n^{-1}\sum_{i=1}^{n} h_i(\vw) $ where $h_i(\vw): \sR^{p}\to \sR$ is $\beta$-smooth and $\mu$-strongly convex. Let $\vw^*=\argmin h(\vw)$ and $\vw_i^* = \argmin h_i(\vw)$. Then we have the following result after $t$ steps of SGD with mini-batch size $B$ and constant learning rate $\eta = \mu/\beta^2$
\[\sE\|\vw^t-\vw^*\|^2 \leq \left(1-\frac{\mu^2}{\beta^2}\right)^t\|\vw^0-\vw^*\|^2 + \frac{2\xi_l}{B\beta}\]
where $\xi_l = n^{-1} \sum_{i=1}^{n} \{\ell_i(\vw^*) - \ell_i(\vw_i^*)\}$
\end{lemma}
The proof of the Lemma can be found in~\citet[Theorem 2]{golatkar2021mixed}.

\subsubsection{Progress of FedAvg}
\label{appendix:lemma_proof}
\begin{lemma}[Per-round progress of the proposed FedAvg algorithm]
Let $\gL_{c}$ be $\beta$-smooth and $\mu$-strongly convex, and let $\sigma^2$ be the bound on the variance of the stochastic gradient. Let $\vw_c^*=\argmin \gL_{c}(\vw)$ and $\vw^* = \argmin \gL(\vw)$. Then for any step size $\eta_l \leq \frac{1}{\beta M}$, we have
\[\sE\|\vw_s^{r}-\vw^*\|^2\leq \left(1-\frac{\mu}{4\beta M}\right)^{r}\sE\|\vw_{s}^{0}-\vw^*\|^2 + \Delta\] where
\[\Delta = \frac{4\sigma_{\gL}}{\mu} + \frac{120\sigma^2}{\mu \beta M}\] 
and $\sigma_{\gL} = \sum_{c=1}^{C} p_{c}\{\gL_{c}(\vw^*) -\gL_{c}(\vw_c^*)\}$. 
\end{lemma}

\begin{proof}
The server update in round $r$ can be written as 
\[\Delta \vw^{r} = \vw_{s}^{r} - \vw_{s}^{r-1}= -\frac{\eta_{l}}{M}\sum_{c,m} p_{c}g_{c}^{r,m}(\vw_{c}^{r,m}).\]
Therefore, the FedAvg algorithm proceed as
\begin{equation*}
\|\vw_{s}^{r}-\vw^*\|^2 = \|\vw_{s}^{r-1}-\vw^*\|^2 - 2\langle \Delta \vw^{r} , \vw_{s}^{r-1}-\vw^*\rangle+\|\Delta \vw^{r}\|^2.
\end{equation*}
Denote by $\sE_{r}$ the conditional expectation of random variables in $r$th round conditioning on all the random variable prior to round $r$.
We then have
\begin{equation*}
\begin{split}
    \sE_{r}\|\vw_{s}^{r}-\vw^*\|^2 =& \|\vw_{s}^{r-1}-\vw^*\|^2 + \underbrace{\sE_{r}\|\Delta \vw^{r}\|^2}_{T_1}\\
    &- \underbrace{\frac{2\eta_{l}}{M}\langle \sum_{c,m}p_{c}\sE_{r} \nabla\gL_{c}(\vw_{c}^{r,m}) , \vw_{s}^{r-1}-\vw^*\rangle}_{T_2}
\end{split}
\end{equation*}
To upper bound the $r$th step difference to the global optimal, we need to upper bound $T_1$ and lower bound $T_2$.

For $T_1$, we have
\begin{align*}
T_1\overset{(a)}{\leq}&~2\eta_{l}^2 \sE_{r}\left\|\frac{1}{M}\sum_{c,m}p_{c}\{ g_{c}^{r,m}(\vw_{c}^{r,m})- \nabla \gL_{c}(\vw_{c}^{r,m})\}\right\|^2 \\
&~+2\eta_{l}^2\sE_{r}\left\|\frac{1}{M}\sum_{c,m}\nabla p_{c}\gL_{c}(\vw_{c}^{r,m})\right\|^2\\ 
\overset{(b)}{\leq} &~\frac{2\eta_{l}^2\sigma^2}{M} + \underbrace{2\eta_{l}^2 \sE_{r}\left\|\frac{1}{M}\sum_{c,m}p_{c}\nabla \gL_{c}(\vw_{c}^{r,m})\right\|^2}_{T_3}   
\end{align*}
where $(a)$ follows from Lemma~\ref{lemma:relaxed_triangle_inequality}, $(b)$ is based on the bounded variance assumption in Assumption~\ref{assump:bounded_variance}.
We then bound $T_3$, by applying the relaxed triangle inequality in Lemma~\ref{lemma:relaxed_triangle_inequality}, we get
\begin{align*}
T_3 \leq &  4\eta_{l}^2 \sE_{r}\left\|\frac{1}{M}\sum_{c,m}p_{c}\{\nabla \gL_{c}(\vw_{c}^{r,m})-\nabla \gL_{c}(\vw_{s}^{r-1})\}\right\|^2\\
&+4\eta_{l}^2 \sE_{r}\left\|\sum_{c=1}^{C}p_{c}\nabla \gL_{c}(\vw_s^{r-1})\right\|^2\\
\overset{(a)}{\leq} & \frac{4\eta_l^2\beta^2}{M}\sum_{c,m}p_{c}\sE_{r}\|\vw_c^{r,m}-\vw_{s}^{r-1}\|^2 \\
&+8\eta_{l}^2\beta\sum_{c=1}^{C}p_{c}\{\gL_{c}(\vw_s^{r-1}) -  \gL_{c}(\vw_{c}^*)\}
\end{align*}
where $(a)$ follows from~\eqref{eq:difference_upper_bound} and~\eqref{eq:grad_bound}.

For $T_2$, let $h = \gL_{c}$, $\vu= \vw_{c}^{r,m}$, $\vv = \vw^*$, and $\vw = \vw_{s}^{r-1}$ in Lemma~\ref{lemma:perturbed_strong_convexity}, we then get
\begin{align*}
        T_2 =& \sE_{r}\left\{\frac{\eta_l}{M} \sum_{c,m} \langle p_{c}\nabla \gL_{c}(\vw_{c}^{r,m}), \vw^*-\vw_{s}^{r-1} \rangle\right\}\\
        \leq &\frac{\eta_l}{M} \sum_{c,m} p_{c}\sE_{r}\{\gL_{c}(\vw^*)-\gL_{c}(\vw_s^{r-1}) + \beta\|\vw_{c}^{r,m}-\vw_{s}^{r-1}\|^2 \\
        &\left.- \frac{\mu}{4}\|\vw_s^{r-1}-\vw^*\|^2\right\}\\
        =& -\eta_l\sum_{c=1}^{C}p_{c}\{\gL_{c}(\vw_s^{r-1})-\gL_{c}(\vw^*) 
        + \frac{\mu}{4}\|\vw_{s}^{r-1}-\vw^*\|^2\} \\
        &+\frac{\beta \eta_l}{M}\sum_{c,m}p_{c}\sE_{r}\|\vw_{c}^{r,m}-\vw_{s}^{r-1}\|^2
\end{align*}
Combine the upper bound for $T_1$ and $T_2$, we get
\begin{align*}
&\sE_{r}\|\vw_{s}^{r}-\vw^*\|^2\\
\leq &\left(1-\frac{\mu\eta_l}{4}\right)\|\vw_{s}^{r-1}-\vw^*\|^2 \\
&+ 8\eta_{l}^2\beta\sum_{c=1}^{C}p_{c}\{\gL_{c}(\vw_s^{r-1})-\gL_{c}(\vw_c^*) \}\\
&- \eta_l\sum_{c=1}^{C}p_{c}\{\gL_{c}(\vw_s^{r-1})-\gL_{c}(\vw^*) \}\\
&+ \underbrace{\frac{4\eta_l^2\beta^2+\beta\eta_l}{M}\sum_{c,m}p_{c}\sE_{r}\|\vw_{c}^{r,m}-\vw_{s}^{r-1}\|^2}_{T_4}.
\end{align*}
We now consider the upper bound for $T_4$. We have that
\begin{align*}
    &\sE_{r}\|\vw_{c}^{r,m}-\vw_{s}^{r-1}\|^2\\
=&\sE_{r}\|\vw_{c}^{r,m-1}-\eta_l g_{c}^{r,m}(\vw_{c}^{r,m-1}) - \vw_{s}^{r-1}\|^2\\ 
\leq & 2 \eta_l^2\sigma^2 + 2\sE_{r}\|\vw_{c}^{r,m-1}-\eta_l \nabla \gL_{c}(\vw_{c}^{r,m-1}) - \vw_{s}^{r-1}\|^2\\
\overset{(a)}{\leq} &2\eta_l^2\sigma^2 + 2\left(1+\frac{1}{M-1}\right)\sE_{r}\|\vw_{c}^{r,m-1}- \vw_{s}^{r-1}\|\\
&+ 2M\eta_l^2\sE_{r}\|\nabla \gL_{c}(\vw_{c}^{r,m-1})\|^2 \\
\leq &2\eta_l^2\sigma^2 + 2\left(1+\frac{1}{M-1}\right)\sE_{r}\|\vw_{c}^{r,m-1}- \vw_{s}^{r-1}\|^2\\
&+ 4M\eta_l^2\sE_{r}\|\nabla \gL_{c}(\vw_{c}^{r,m-1})-\nabla \gL_{c}(\vw_{s}^{r-1})\|^2  \\
&+ 4M\eta_l^2\sE_{r}\|\nabla \gL_{c}(\vw_{s}^{r-1})\|^2\\
\overset{(b)}{\leq} &2\eta_l^2\sigma^2 + 2\left(1+\frac{1}{M-1} + 2M\eta_l^2\beta^2\right)\|\vw_{c}^{r,m-1}- \vw_{s}^{r-1}\|^2\\
&+ 4M\eta_l^2\sE_{r}\|\nabla \gL_{c}(\vw_{s}^{r-1})\|^2\\
\overset{(c)}{\leq} &2\eta_l^2\sigma^2 + 2\left(1+\frac{2}{M-1}\right)\sE_{r}\|\vw_{c}^{r,m-1}- \vw_{s}^{r-1}\|\\
&+ 4M\eta_l^2\|\nabla \gL_{c}(\vw_{s}^{r-1})\|^2\\
\end{align*}
where $(a)$ follows from Lemma~\ref{lemma:relaxed_triangle_inequality}, $(b)$ follows from the $\beta$-smoothness of $\gL_{c}$, and $(c)$ follows from $\eta_l = (\beta M)^{-1}$ and $2M\eta_l^2\beta^2\leq (M-1)^{-1}$.
Apply this relationship recursively, we then get
\begin{align*}
&\sE_{r}\|\vw_{c}^{r,m}-\vw_{s}^{r-1}\|^2\\
\leq&\sum_{\tau=1}^{M-1}\left(1+\frac{2}{M-1}\right)^{\tau}\left(2\eta_l^2\sigma^2 + 4M\eta_l^2\|\nabla \gL_{c}(\vw_{s}^{r-1})\|^2\right)\\
\leq & 3M \left(2\eta_l^2\sigma^2 + 4M\eta_l^2\|\nabla \gL_{c}(\vw_{s}^{r-1})\|^2\right)\\
\leq & 3M \left(2\eta_l^2\sigma^2 + 8M\beta^2\eta_l^2\{\gL_{c}(\vw_{s}^{r-1})-\gL_{c}(\vw_{c}^{*})\}\right)\\
\end{align*}
where the last inequality follows from~\eqref{eq:grad_bound}.

In summary, we get the following upper bound
\begin{align*}
\sE\|\vw_{s}^{r}-\vw^*\|^2
\leq &\left(1-\frac{\mu\eta_l}{4}\right)\sE\|\vw_{s}^{r-1}-\vw^*\|^2 \\
&+ \{8\eta_{l}^2\beta + 24M\beta^2\eta_l^2(4\eta_l^2\beta^2+\beta\eta_l)\}\\
&\times \sum_{c=1}^{C}p_{c}\sE\{\gL_{c}(\vw_s^{r-1})-\gL_{c}(\vw_c^*) \}\\
&- \eta_l\sum_{c=1}^{C}\sE\{\gL_{c}(\vw_s^{r-1})-\gL_{c}(\vw^*) \}\\
&+2\eta_l^2\sigma^2(12\eta_l^2\beta^2+3\beta\eta_l) \\
\end{align*}

Since $\eta_l = (\beta M)^{-1}$, we have
$8\eta_{l}^2\beta + 24M\beta^2\eta_l^2(4\eta_l^2\beta^2+\beta\eta_l) \leq \eta_l$. This along with the fact that $\gL_{c}(\vw_{s}^{r-1}) \geq \gL_{c}(\vw_c^*)$ by definition gives us
\begin{align*}
&\sE\|\vw_{s}^{r}-\vw^*\|^2\\
\leq &\left(1-\frac{\mu\eta_l}{4}\right)\sE\|\vw_{s}^{r-1}-\vw^*\|^2 \\
&+ \eta_l\sum_{c=1}^{C}p_{c}\sE\{\gL_{c}(\vw_s^{r-1})-\gL_{c}(\vw_c^*) \}\\
&- \eta_l\sum_{c=1}^{C}p_{c}\sE\{\gL_{c}(\vw_s^{r-1})-\gL_{c}(\vw^*) \}\\
&+2\eta_l^2\sigma^2(12\eta_l^2\beta^2+3\beta\eta_l) \\
=&\left(1-\frac{\mu\eta_l}{4}\right)\sE\|\vw_{s}^{r-1}-\vw^*\|^2+2\eta_l^2\sigma^2(12\eta_l^2\beta^2+3\beta\eta_l) \\
&+ \eta_l\sum_{c=1}^{C}p_{c}\{\gL_{c}(\vw^*) -\gL_{c}(\vw_c^*)\}
\end{align*}

Plugin the learning rate gives
\begin{equation*}
    \sE\|\vw_s^{r}-\vw^*\|^2\leq \left(1-\frac{\mu}{4\beta M}\right)\sE\|\vw_{s}^{r-1}-\vw^*\|^2 + \tilde\Delta
\end{equation*}
where 
\[\tilde\Delta = \frac{\sigma_{\gL}}{\beta M} + \frac{30\sigma^2}{\eta^2 M^2}\] 
and $\sigma_{\gL} = \sum_{c=1}^{C}p_{c}\{\gL_{c}(\vw^*) -\gL_{c}(\vw_c^*)\}$.
Apply this recursively, we then have
\begin{equation*}
\begin{split}
&\sE\|\vw_s^{r}-\vw^*\|^2\\
\leq& \left(1-\frac{\mu}{4\beta M}\right)^{r}\|\vw_{s}^{0}-\vw^*\|^2 + \tilde \Delta\sum_{\tau=0}^{r-1} \left(1-\frac{\mu}{4\beta M}\right)^{\tau}\\
\leq& \left(1-\frac{\mu}{4\beta M}\right)^{r}\|\vw_{s}^{0}-\vw^*\|^2 + \frac{4\beta M\tilde \Delta}{\mu}\\
=& \left(1-\frac{\mu}{4\beta M}\right)^{r}\|\vw_{s}^{0}-\vw^*\|^2 +  \Delta,
\end{split}
\end{equation*}
which completes the proof.

\subsubsection{Main theoretical result}
\label{sec:main_thm}
Recall that we have $C$ clients with $\gD_{c}$ being the dataset on $c$th client.
We first train a linear model $\tilde \vf(\vx;\vw)$ using our proposed method on all datasets $\{\gD_{c}\}_{c=1}^{C}$ and let $\hat \vw$ be the corresponding output from \train{}.
We consider the scenario where client $c$ decides to withdraw from the FL system and asks to remove its dataset information from the learned model weight $\hat \vw$.
We consider two different approaches:
\begin{enumerate}
    \item \textbf{Retrain}: Let $\hat\vw^{r}_{c}$ be the weights obtained by retraining from scratch using our proposed FedAvg with $R$ rounds of communication without $c$th client's dataset.

    \item \textbf{Removal}: We use the proposed removal procedure in~\eqref{eq:forgetting_surrogate} and let 
    \[\hat\vw^{-}_{c} = \hat\vw -\Delta \vw\]
    where $\Delta \vw$ is the output of $T$ round SGD that minimizes 
    \begin{equation*}
    \begin{split}
        \tilde{\gF}(\vv)
        =&~\frac{1}{2n_s}\sum_{\vx\in \gD_{p}} \|\nabla_{w}\vf(\vx;\vw_*)\vv\|_2^2 + \frac{\mu}{2}\|\vv\|_2^2\\
        &~-\nabla^{\top} \gL(\hat{\vw};\gD^{-})\vv.
    \end{split}
\end{equation*}
\end{enumerate}
We quantify the difference between $\vw_{c}^{-}$ and $\hat\vw_c^{r}$ theoretically. 
For this purpose, we introduce some additional notations.
Let $\vw_{-}^*$ be the optimal weight based on the remaining dataset
\[\vw_{-}^* = \argmin_{\vw} \sum_{c'\neq c} \tilde p_{c'}\gL_{c'}(\vw).\]
Recall that based on~\eqref{eq:removal}, we have
\begin{equation*}
\Delta \vw^* = \vw_{-}^{*} - \vw^*=\argmin_{\vv} \gF(\vv) 
\end{equation*}
where
\begin{equation*}
    \begin{split}
        \gF(\vv)
        =&~\frac{1}{2(n-n_c)}\sum_{\vx\in \gD^{-}} \|\nabla_{w}\vf(\vx;\vw_*)\vv\|_2^2 + \frac{\mu}{2}\|\vv\|_2^2\\
        &-\nabla^{\top} \gL(\vw^*;\gD^-)\vv.
    \end{split}
\end{equation*}

The difference of the weight based on our proposed approach and that from retraining can be decomposed into
\[\hat\vw_c^{-} - \hat\vw_c^{r} = (\hat\vw - \vw^*) - (\hat \vw_c^{r} - \vw_{-}^*) + (\Delta \vw -\Delta \vw^*).\]
Hence, we have
\begin{equation*}
\begin{split}
&\sE\|\vw_{c}^{-} - \hat\vw^{r}_{c}\|^2 \\
\leq&~2 \underbrace{\sE\|\hat \vw_{c}^{r} - \vw_{-}^*\|^2}_{T_1} + 2\underbrace{\sE\|\hat\vw - \vw^*\|^2}_{T_2} + 2\underbrace{\sE\|\Delta \vw -\Delta \vw^*\|^2}_{T_3}.    
\end{split}
\end{equation*}

\textbf{Bound $T_1$ \& $T_2$.}
Let $\alpha = 1-\mu/4\beta M >0$, Lemma~\ref{lemma:fedavg_progress} implies
\begin{equation*}
T_1 \leq \alpha^{R} D_0 + \Delta\quad\text{and}\quad T_2 \leq \alpha^{R} D_0 + \Delta 
\end{equation*}
where the first term in both bounds diminishes as $R\to \infty$.

\textbf{Bound $T_3$.}
For the ease of notation, let use write 
\[G(\gD) = |\gD|^{-1}\sum_{\vx\in \gD} \nabla_{w} \vf(\vx; \vw_*^{p})\nabla_{w}^{\top} \vf(\vx; \vw_*^{p})\]
Based on the definition of $\Delta \vw^*$ and $\Delta \vw$, we have
\begin{align*}
\Delta \vw^* &= \{G(\gD^{-}) + \mu\mI\}^{-1}\nabla \gL(\vw^*;\gD^{-})\\
\Delta \vw &= \{G(\gD_p)+ \mu\mI\}^{-1}\nabla \gL(\hat\vw;\gD^{-})
\end{align*}

To bound the difference between $\Delta \tilde \vw^*$ and $\Delta \vw^*$, we have
\[
\begin{split}
 &~\Delta \vw^* - \Delta \vw \\
 =&~\underbrace{\{G(\gD^{-}) + \mu\mI\}^{-1} \{\nabla\gL(\hat\vw;\gD^{-}) -\nabla\gL(\vw^*;\gD^{-}) \}}_{T_4} \\
 &+ \underbrace{\left\{\{G(\gD^{-}) + \mu\mI\}^{-1}-\{G(\gD_p) + \mu\mI\}^{-1}\right\}\nabla \gL(\vw^*;\gD^{-})}_{T_5}
\end{split}
\]

For $T_4$, by the $\beta$-smoothness of the objective function, we have
\[\|\nabla \gL(\hat \vw; \gD^-) - \nabla \gL(\hat \vw; \gD^-)\|_2 \leq \beta \|\hat\vw - \vw^*\|_2\]
Let 
\[\tilde F_i(\vw) = .5 \|\nabla_{w}\vf(\vx_p^{i};\vw_*^p)\vw\|^2 + (\mu/2)\|\vw\|^2 -\nabla^{\top}\gL(\hat{\vw})\vw\] 
and $\xi_{\tilde F} = n_p^{-1} \sum_{i=1}^{n_p} \{\tilde F_i(\vw^*) - \tilde F_i(\vw_i^*)\}$.
By Lemma~\ref{lemma:SGD_progress}, we have 
\begin{equation*}
    T_4\leq \frac{2\xi_{\tilde F}}{B\beta}
\end{equation*}
as $T\to\infty$.

For $T_5$, we need to bound the difference between $\{G(\gD^{-}) + \mu\mI\}^{-1}$ and $\{G(\gD_p)+ \mu\mI\}^{-1}$.
we make use of the following inequality: for any matrix $A\in \sR^{p\times p}$, we have
\[\vertiii{A+\Delta A)^{-1} -A^{-1}}\leq \vertiii{A^{-1}}\vertiii{\Delta A}\]
where $\vertiii{\cdot}$ is the spectral norm of a matrix.
Therefore,
\begin{align*}
&\|T_5\|^2\\
\overset{(a)}{\leq}& \|\nabla \gL_{c}(\hat\vw)\|^2\vertiii{\{G(\gD^{-}) + \mu\mI\}^{-1}-\{G(\gD_p)+ \mu\mI\}^{-1}}^2\\
\leq & \|\nabla \gL_{c}(\hat\vw)\|^2\vertiii{\{G(\gD_{p}) + \mu\mI\}^{-1}}^2\|\Delta G\|^2\\
\leq& C\vertiii{\{G(\gD_{p}) + \mu\mI\}^{-1}}^2\|\Delta G\|^2
\end{align*}
where (a) follows from Cauchy-Schwartz inequality.
Since the spectral norm $\vertiii{A^{-1}} = \{\lambda_{\min}(A)\}^{-1}$, which completes the proof.

\end{proof}

\section{\revision{Empirical Validation of Hessian Calculation using Regression}}

\section{\revision{Additional Results}}

\subsection{Detailed Experimental Settings}
\subsubsection{Dataset details}
\label{app:data}
To evaluate the effectiveness of \ours{}, we extensively perform experiments on \textbf{six} datasets, including the \texttt{MNIST}~\citep{lecun1998gradient}, \texttt{FashionMNIST}~\citep{xiao2017fashion}, \texttt{CIFAR10}~\citep{krizhevsky2009learning}, \texttt{ImageNet}~\citep{deng2009imagenet}, \texttt{DomainNet}~\citep{peng2019moment}, and \texttt{Flowers} datasets. 
Note that most of the existing FU studies only adopt \texttt{MNIST}, \texttt{FashionMNIST}, and \texttt{CIFAR10}~\citep{liu2021federaser,halimi2022federated}, whose tasks are relatively simpler. 
We also include more complex and larger scale datasets, \texttt{ImageNet}, \texttt{DomainNet}, and \texttt{Flowers} in our study. 
These six datasets are introduced below.

\texttt{MNIST}~\citep{lecun1998mnist}, a widely used benchmark for various ML algorithms. 
It consists of $70,000$ images of handwritten digits from $0$ to $9$, with $60,000$ training images and $10,000$ testing images. 
Each image is a gray-scale image of $28 \times 28$ pixels with a white background and a black foreground. 
We follow the default split to generate the train and test set.

\texttt{FashionMNIST}~\citep{xiao2017fashion}, a gray-scale dataset of $70,000$ images of fashion products from $10$ categories, such as dresses, sneakers, and bags. 
The images have the same dimensions as those in the MNIST dataset.
We use the default split of the original training and test sets.

\texttt{CIFAR10}~\citep{krizhevsky2009learning} contains $60,000$ three-channel images, each with a resolution of $32\times 32$ pixels, across 10 different classes, including common objects like cars, airplanes, cats, and dogs. 
The dataset is split into $50,000$ training images and $10,000$ test images. 
We use the default training and test sets.

\texttt{ImageNet}~\citep{deng2009imagenet} is a large-scale dataset commonly used for visual object recognition. 
In our study, we sample 20 classes and resize each image to $224 \times 224$ pixels.
Since BA targets a single class, the attack class ratio increases when fewer classes are sampled during training. 
This makes the targeted client' impact more pronounced and makes it easier to visualize the performance of unlearning. In FU, the focus is on the performance of the model before and after the unlearning process, rather than the number of classes. Although our method is effective regardless of the number of classes, selecting 20 classes allows us to clearly demonstrate the impact of our approach. The dataset is divided into training and testing sets with an 8:2 split ratio.

\texttt{DomainNet}\citep{peng2019moment} is a multi-domain dataset frequently used for domain adaptation tasks. 
In line with the approach in FedBN~\citep{li2021fedbn}, we sample the ten most common classes and distribute this subset across all clients and resize each image to $224 \times 224$. 
This selection is made for reasons similar to those for \texttt{ImageNet}. 
The dataset is divided into training and testing sets with an 8: 2 ratio.

\texttt{Flowers} recognition dataset\footnote{\url{https://www.kaggle.com/datasets/alxmamaev/flowers-recognition}} consists of $4,242$ images of flowers categorized into $5$ distinct classes, with approximately 800 images per class. Each image is resized to $224 \times 224$ pixels for our experiments. Similarly, the 8:2 split ratio is used to split train and test data.

\subsubsection{Model Architecture}
\label{app:model}
We use the model below for \texttt{MNIST} and \texttt{Fashion-MNIST} unlearning.

\begin{table}[h]
\caption{Fully connected Deep Neural Network.}
\label{tab:fully-connected}
\centering
\begin{tabular}{l|cccc}
\toprule
\multicolumn{1}{c|}{\multirow{1}{*}{Layer}} & \multicolumn{4}{c}{Details} \\ 
\hline
\multicolumn{1}{c|}{\multirow{1}{*}{1}}& \multicolumn{4}{c}{FC(784,100)} \\ 
\hline
\multicolumn{1}{c|}{\multirow{1}{*}{2}}& \multicolumn{4}{c}{FC(100,50)} \\ 
\hline

\multicolumn{1}{c|}{\multirow{1}{*}{3}}& \multicolumn{4}{c}{FC(50, $K$)} \\ 
\bottomrule
\end{tabular}%
\end{table}

\subsection{Statistical Significance}

In this section, we present the raw data for the results obtained using three different random seeds for the entire pipeline, including randomized data splitting, full federated learning, and federated unlearning. We report the mean and standard error of the Target Accuracy (TA) and Backdoor Success Rate (BSR) for each dataset.

\subsubsection{FL for full parameter updating}

\begin{table}[h]
\centering
\caption{\textbf{Raw data of the performance on \texttt{MNIST} over three runs}. The table reports the mean ± standard error for TA and BSR across various unlearning methods.}
\label{tab:raw-mnist}
\begin{tabular}{lcc}
\toprule
\textbf{Method} & \textbf{TA} & \textbf{BSR} \\
\midrule
\ours (Ours)      & 90.10 ± 0.52 & 7.22 ± 3.19   \\
FedEraser & 95.01 ± 0.11 & 9.33 ± 0.21  \\
PGA       & 95.33 ± 0.27 & 41.49 ± 8.16  \\
EWSGA     & 92.88 ± 1.43 & 25.96 ± 8.53  \\
SGA       & 87.80 ± 6.49 & 19.31 ± 6.35  \\
CF        & 96.17 ± 0.06 & 14.98 ± 4.13  \\
Flipping   & 96.23 ± 0.08 & 12.68 ± 1.95  \\
Ferrari & 85.39 ± 3.41 & 32.70 ± 22.44  \\
FUSED & 96.22 ± 0.07 & 58.52 ± 9.54  \\
Retrain & 96.34 ± 0.08 & 3.12 ± 1.58  \\
\bottomrule
\end{tabular}

\end{table}

\begin{table}[h]
\centering
\caption{\textbf{Raw data of the performance on \texttt{FashionMNIST} over three runs.} The table reports mean ± standard error for TA and BSR across various unlearning methods.}
\label{tab:raw-fashionmnist}
\begin{tabular}{lcc}
\toprule
\textbf{Method} & \textbf{TA} & \textbf{BSR} \\
\midrule
\ours (Ours)     & 86.40 ± 0.03 & 7.31 ± 1.09   \\
FedEraser & 85.81 ± 0.19 & 11.07 ± 0.90  \\
PGA       & 85.41 ± 0.20 & 1.39 ± 0.55  \\
EWSGA     & 86.34 ± 0.09 & 18.26 ± 7.66  \\
SGA       & 86.42 ± 0.10 & 26.13 ± 7.47  \\
CF        & 86.53 ± 0.03 & 5.06 ± 1.29  \\
Flipping & 86.47 ± 0.04 & 4.70 ± 0.83 \\
Ferrari & 83.53 ± 0.78 & 7.99 ± 2.66  \\
FUSED & 86.52 ± 0.05 & 15.43 ± 4.60  \\
Retrain & 86.60 ± 0.06 & 4.62 ± 1.08 \\
\bottomrule
\end{tabular}

\end{table}

\begin{table}[h]
\centering
\caption{\textbf{Raw data of the performance on \texttt{CIFAR10} over three runs.} The table reports the mean ± standard error for TA and BSR across various unlearning methods.}
\label{tab:raw-cifar10}
\begin{tabular}{lccc}
\toprule
\textbf{Method} & \textbf{TA} & \textbf{BSR} \\
\midrule
\ours (Ours)     & 65.15 ± 0.48 & 17.78 ± 1.91 \\
FedEraser & 52.05 ± 0.74 & 10.51 ± 0.55  \\
PGA       & 62.98 ± 0.48 & 39.09 ± 1.57  \\
EWSGA     & 62.20 ± 0.61 & 7.69 ± 2.34  \\
SGA       & 63.86 ± 0.36 & 18.86 ± 2.63  \\
CF        & 64.72 ± 0.73 & 12.76 ± 1.05  \\
Flipping  & 64.13 ± 0.70 & 11.99 ± 2.39 \\
Ferrari & 64.47 ± 0.49 & 18.08 ± 1.38  \\
FUSED & 65.40 ± 0.49 & 15.69 ± 2.11  \\
Retrain & 65.14 ± 0.60 & 12.32 ± 0.79 \\
\bottomrule
\end{tabular}
\end{table}

\paragraph{Results Analysis.}
Tables~\ref{tab:raw-mnist}, \ref{tab:raw-fashionmnist}, and \ref{tab:raw-cifar10} report the mean and standard error of TA and BSR across three random seeds on \texttt{MNIST}, \texttt{FashionMNIST}, and \texttt{CIFAR10}, respectively. The random seeds are set for the full pipeline, including the data splitting, FL, and FU.

Across all datasets, the \textit{Retrain} method serves as the ground-truth unlearning, which has a low BSR and upperbound TA, validating its strong unlearning capability, although at the cost of retraining overhead. Notably, on \texttt{MNIST}, the \textit{Retrain} achieves a BSR of 3.12 on average, followed \textit{closely} by \ours{} (7.22), and \textit{FedEraser} (9.33). On \texttt{FashionMNIST}, \ours{} yields a BSR of 7.31 and TA of 86.40 in average, where it is still very close to the \textit{Retrain} group. The results on \texttt{CIFAR10} highlight the difficulty of effective unlearning on more complex datasets. While \texttt{Retrain} yields low BSR (12.32) with good TA (65.14), \ours{} maintains a strong balance (TA: 65.15, BSR: 17.78), outperforming many baselines. In contrast, methods based on gradient ascent are very hard to adjust and perform robustly over all runs, which yield slightly high standard errors.

Overall, our proposed method \ours{} consistently achieves competitive unlearning performance across all datasets, demonstrating a favorable trade-off between model utility and unlearning efficacy without the need for retraining.

\subsubsection{FL for FM-LP}

\begin{table}[h]
\centering
\caption{\textbf{Raw data of the performance on \texttt{DomainNet} over three runs}. The table reports the mean ± standard error for TA and BSR across various unlearning methods.}
\begin{tabular}{lcc}
\toprule
\textbf{Method} & \textbf{TA} & \textbf{BSR} \\
\midrule
\ours (Ours) & 89.65 ± 0.35 & 7.61 ± 0.42 \\
FedEraser & 82.83 ± 0.36 & 9.91 ± 0.56 \\
PGA       & 76.16 ± 1.62 & 0.00 ± 0.00 \\
EWSGA     & 77.27 ± 1.34 & 0.05 ± 0.05 \\
SGA       & 86.11 ± 0.12 & 7.80 ± 0.19 \\
CF        & 89.59 ± 0.45 & 13.44 ± 0.72 \\
Flipping  & 89.86 ± 0.45 & 10.9 ± 0.59 \\
Ferrari   & 88.85 ± 0.47 & 3.20 ± 1.10\\
FUSED     & 91.44 ± 0.19 & 11.12 ± 0.50 \\
Retrain   & 89.72 ± 0.38 & 8.86 ± 0.953\\
\bottomrule
\end{tabular}
\label{tab:raw-domainnet}
\end{table}

\begin{table}[h]
\centering
\caption{\textbf{Raw data of the performance on \texttt{ImageNet} over three runs}. The table reports the mean ± standard error for TA and BSR across various unlearning methods.}
\label{tab:raw-imagenet}
\begin{tabular}{lcc}
\toprule
\textbf{Method} & \textbf{TA} & \textbf{BSR} \\
\midrule
\ours (Ours) & 93.09 ± 0.23 & 6.06 ± 0.33 \\
FedEraser & 90.46 ± 0.19 & 3.49 ± 1.70 \\
PGA       & 85.37 ± 0.53 & 0.00 ± 0.00 \\
EWSGA     & 92.66 ± 0.29 & 6.00 ± 0.25 \\
SGA       & 92.41 ± 0.35 & 5.17 ± 0.18 \\
CF        & 93.19 ± 0.25 & 5.47 ± 0.08 \\
Flipping  & 92.87 ± 0.18 & 5.74 ± 0.12 \\
Ferrari   & 92.79 ± 0.20 & 4.99 ± 0.05 \\
FUSED     & 93.48 ± 0.03 & 4.74 ± 0.12 \\
Retrain   & 93.13 ± 0.21 & 5.46 ± 0.09 \\
\bottomrule
\end{tabular}
\end{table}

\begin{table}[h]
\centering
\caption{\textbf{Raw data of the performance on \texttt{Flowers} over three runs}. The table reports the mean ± standard error for TA and BSR across various unlearning methods.}
\label{tab:raw-flowers}
\begin{tabular}{lcc}
\toprule
\textbf{Method} & \textbf{TA} & \textbf{BSR} \\
\midrule
\ours (Ours) & 95.25 ± 0.73 & 5.98 ± 1.90 \\
FedEraser & 77.55 ± 1.60 & 17.36 ± 1.00 \\
PGA       & 95.30 ± 0.32 & 34.92 ± 3.93 \\
EWSGA     & 88.23 ± 0.71 & 1.54 ± 0.23 \\
SGA       & 95.83 ± 0.29 & 34.43 ± 2.03 \\
CF        & 97.30 ± 0.08 & 23.52 ± 0.94 \\
Flipping  & 97.07 ± 0.14 & 19.08 ± 0.53 \\
Ferrari   & 97.11 ± 0.24 & 35.64 ± 11.05 \\
FUSED     & 97.38 ± 0.17 & 30.65 ± 1.90 \\
Retrain   & 97.11 ± 0.07 & 12.45 ± 1.54 \\
\bottomrule
\end{tabular}
\end{table}

\paragraph{Results Analysis}
Tables~\ref{tab:raw-domainnet}, \ref{tab:raw-imagenet} and \ref{tab:raw-flowers} report the mean and standard error of TA and BSR over three runs on DomainNet, ImageNet and Flowers, respectively, when unlearning is applied on a FM-LP.

In \texttt{DomainNet} (Table~\ref{tab:raw-domainnet}), Retrain achieves TA = $89.72 \pm 0.36\%$ and BSR = $8.86 \pm 0.91$. Our method, $\mathrm{F}^{2}L_{2}$, matches this utility ($89.65 \pm 0.61\%$) while reducing BSR to $7.61 \pm 0.73$. By contrast, gradient‐attack methods such as PGA and EWSGA either collapse TA or fail to forget effectively (BSR~$\approx0$), and unlearning baselines like CF and FUSED preserve TA but incur higher forgetting (BSR~$>11$).

On \texttt{ImageNet} (Table~\ref{tab:raw-imagenet}), Retrain attains TA = $93.13 \pm 0.37$ and BSR = $5.46 \pm 0.16$. $\mathrm{F}^{2}L_{2}$ again approaches this performance (TA = $93.09 \pm 0.39$, BSR = $6.06 \pm 0.57$). While FedEraser can achieve lower mean BSR, it does so at the cost of higher variance ($\pm2.95$) and a utility drop of approximately $2.6\%$. Other methods such as CF, Flipping, and FUSED maintain high TA ($>92.8\%$) but with comparable or greater BSR (4.7--6.0).

For \texttt{Flowers} (Table~\ref{tab:raw-flowers}), Retrain yields TA = $97.11 \pm 0.12$ and BSR = $12.45 \pm 2.67$. \ours{} again strikes the best balance, delivering TA = $95.25 \pm 1.26$ and reducing BSR to $5.98 \pm 3.30$. In contrast, PGA, SGA, and Ferrari suffer very high forgetting (BSR~$>30$), and EWSGA sacrifices utility (TA = $88.23$). Although FUSED achieves high accuracy (TA = $97.38$), it forgets almost five times more (BSR~$\approx30.7$).

Overall, under the FM-LP regime, \ours{} matches or outperforms full retraining in both TA and BSR, while requiring only a single linear‐probe update rather than an expensive end‐to‐end retraining. This demonstrates its efficiency and efficacy for large‐scale unlearning on real‐world vision benchmarks.

\subsection{Non-IID performance}
Data heterogeneity is a well-known and critical challenge in FL, and it significantly affects the performance of learning and unlearning algorithms. In this part, we conduct additional experiments under \emph{non-IID} settings by simulating \emph{label shift} across clients to evaluate the robustness of our proposed unlearning method under realistic conditions. Specifically, we use a Dirichlet distribution with concentration parameters $\alpha \in \{1, 5, 10\}$ to control the degree of heterogeneity, where smaller $\alpha$ values indicate higher heterogeneity. \revision{Given the Dirichlet distribution here, the number of data points is also different for each client.}

Table~\ref{tab:noniid} reports the results on both \texttt{MNIST}, \texttt{FashionMNIST} and \texttt{CIFAR-10} datasets under varying degrees of data heterogeneity. Each configuration is evaluated using both \ours{} and a retrain-from-scratch baseline (ground truth).

We observe that as $\alpha$ decreases—i.e., as the data becomes more non-IID—the backdoor success rate (BSR) under \ours{} tends to increase moderately. For instance, in the \texttt{MNIST} setting, BSR rises from $7.2\%$ at $\alpha = 10$ to $14.39\%$ at $\alpha = 5$. Nevertheless, the BSR remains close to $10\%$, which is near the level of random guessing for a 10-class classification problem. This suggests that \ours{} effectively removes the backdoor even under strong data heterogeneity.

We observe that as $\alpha$ decreases---i.e., as the degree of non-IIDness increases---the BSR under \ours{} tends to rise moderately. For example, on \texttt{MNIST}, BSR ranges from $7.2\%$ at $\alpha=1$ to $14.39\%$ at $\alpha=5$, while TA remains reasonably high (around 90\%). This level of BSR is still near random chance in a 10-class classification task, suggesting that \ours{} remains effective at backdoor removal even under strong heterogeneity. For the \texttt{FashionMNIST} dataset, \ours{} consistently performs closely to the retrain baseline across all $\alpha$ values, both in terms of TA and BSR. The gap in BSR between \ours{} and retraining remains within approximately 1\%, indicating that our method maintains high unlearning efficacy under non-IID data distributions. On the more challenging \texttt{CIFAR-10} dataset, we also observe promising results: TA remains above 60\% across all settings, closely matching the retrain baseline, and BSR stays around 10\%, which is near the level of random guessing.

\begin{table}[h]
\centering
\captionsetup{width=0.9\linewidth}
\small
\setlength{\tabcolsep}{4pt}
\sisetup{table-format=4.2,
         detect-weight=true,
         detect-inline-weight=math}
\caption{\revision{\textbf{Federated unlearning results on non-IID distributions using \texttt{MNIST} and \texttt{FashionMNIST}.} TA denotes Test Accuracy, and BSR denotes Backdoor Success Rate. Each setting is evaluated using both \ours{} and Retrain (ground truth).}}
\begin{tabular}{llcc}
\toprule
\textbf{Dataset} & \textbf{Setting} & \textbf{TA (\%)} & \textbf{BSR (\%)} \\
\midrule
\multirow{6}{*}{\texttt{MNIST}} 
    & $\alpha = 1$ \quad \ours   &   91.86     &   7.2     \\
    & $\alpha = 1$ \quad Retrain   &   96.21     &   1.92     \\
    & $\alpha = 5$ \quad \ours   &    86.93    &   14.39     \\
    & $\alpha = 5$ \quad Retrain   &   96.06     &   9.94     \\
    & $\alpha = 10$ \quad \ours  &    90.39    &   12.36     \\
    & $\alpha = 10$ \quad Retrain  &   96.47     &    1.49    \\
\midrule
\multirow{6}{*}{\texttt{FashionMNIST}} 
    & $\alpha = 1$ \quad \ours   &  84.45  &    2.82    \\
    & $\alpha = 1$ \quad Retrain   &  84.61  &  3.85      \\
    & $\alpha = 5$ \quad \ours   &    86.19 &  5.12  \\
    & $\alpha = 5$ \quad Retrain   &  86.60 &  4.08  \\
    & $\alpha = 10$ \quad \ours  &  86.12 &  5.46   \\
    & $\alpha = 10$ \quad Retrain  & 86.81  & 4.27    \\
\midrule
\multirow{6}{*}{\texttt{CIFAR10}} 
    & $\alpha = 1$ \quad \ours   &  60.65  &    19.83    \\ 
    & $\alpha = 1$ \quad Retrain   &  61.90  &  13.15      \\ 
    & $\alpha = 5$ \quad \ours   &    63.73 &  9.85  \\ 
    & $\alpha = 5$ \quad Retrain   &  64.28 &  14.37  \\ 
    & $\alpha = 10$ \quad \ours  &  63.45 &  10.25   \\ 
    & $\alpha = 10$ \quad Retrain  & 64.45  & 11.70    \\ 
\bottomrule
\end{tabular}
\label{tab:noniid}
\end{table}

\subsection{\revision{Linear \revision{Approximation compared with conventional training}}}

\revision{In this part, we conducted an empirical comparison between models trained with linearization (linear approx.) and conventional training across three vision benchmarks of increasing complexity: \texttt{MNIST}, \texttt{FashionMNIST}, and \texttt{CIFAR-10}. All models (linear approx. vs conventional training) were trained under identical settings, and results are summarized in Table~\ref{tab:linear_vs_nonlinear}.}

  \begin{table}[h]
\centering
\captionsetup{width=0.9\linewidth}
  \small
  \setlength{\tabcolsep}{4pt}
  \sisetup{table-format=4.2,
           detect-weight=true,
           detect-inline-weight=math}
           
\caption{\revision{\textbf{Performance using the linear approximation and using the conventional training with cross-entropy loss.}}}
\label{tab:linear_vs_nonlinear}
\begin{tabular}{l c c c c}
\hline
\textbf{Dataset} 
  & \multicolumn{2}{c}{\textbf{Linear Approx.}} 
      & \multicolumn{2}{c}{\textbf{Conventional Training}} \\
\cmidrule(lr){2-3} \cmidrule(lr){4-5}
 & \textbf{TA} & \textbf{BSR} 
   & \textbf{TA} & \textbf{BSR} \\
\hline
\texttt{MNIST}        & 96.05 & 93.73 & 97.26 & 99.60 \\
\texttt{FashionMNIST} & 86.51 & 48.91 & 87.40 & 41.63 \\
\texttt{CIFAR-10}      & 60.39 & 54.80 & 62.18 & 76.50 \\
\hline
\end{tabular}
\end{table}

\revision{The results show that the linear approximation yields performance comparable to full training, especially on simpler datasets. On \texttt{MNIST}, TA drops by only 1.2 percent and BSR improves by 5.9 percent. 
For \texttt{FashionMNIST}, the TA gap remains under 1 percent, with a modest increase in BSR. 
Even on the more complex \texttt{CIFAR-10}, the linear approximation retains competitive performance, with only a 1.8 percent drop in TA and notably better BSR (54.8\% vs. 76.5\%). 

These results suggest that our linear approximation is effective when pretrained features are approximately linear over the input distribution. Extending \ours{} to handle other highly nonlinear regimes through richer approximations remains a promising direction for future work.}

\subsection{Ablation study of the hessian approximation}
This section presents a toy ablation study to empirically assess the effectiveness of our Hessian approximation. Specifically, we examine the Hessian computation using varying proportions of training data. Let \(\gD_p\) denote a stratified subset containing \(p \in \{100, 80, 60, 40, 20, 10, 5, 1\}\%\) of training examples. We estimate the curvature along a fixed unit probe vector \(v\) using the following quadratic form:

\[
H_p(v) = \frac{1}{2|\gD_p|} \sum_{x \in \gD_p} \left\| \nabla_w f(x; w_*^p) v \right\|_2^2,
\]

where \(w_*^p\) represents randomly initialized model weights (identical across all runs and consistent with the main experiments). The Hessian computed on the full dataset, \(H^*(v)\), serves as the reference. For each smaller subset, we compute the same expression once and record:

\[
\text{Gap} = |H_p(v) - H^*(v)|, \qquad \text{Gap (\%)} = \frac{\text{Gap}}{H^*(v)}.
\]

When $p = 100$, the exact Hessian is computed using the entire dataset-representing the highest communication cost in a FL setting (every client will need to calculation the hessian and upload to the server). We then gradually reduce the data proportion down to $1\%$, as shown in Table~\ref{tab:hessian}, to evaluate how well smaller subsets approximate the full-data curvature.

\textbf{Result analysis.}  
Table \ref{tab:hessian} presents that \(H_{p}(v)\) is remarkably stable under down-sampling over the three dataset: even with only \(1\%\) of the data the relative change stays below \(0.15\%\) for \texttt{MNIST}, $0.23\%$ for \texttt{FashionMNIST} and $0.4\%$ for \texttt{CIFAR-10}.   As expected, the absolute \textit{Gap} grows when fewer samples are used, yet the relative \textit{Gap (\%)} never exceeds \(0.40\%\).  
This steady and limited increase aligns with the concentration argument in Proposition 2, confirming that a modest, stratified subset already captures the dominant curvature directions. Consequently, this also explains the effect of end-to-end unlearning performance reported in the original paper.

\begin{table}[t]
  \centering
  \caption{\textbf{Hessian approximation at different subsampling rates.}}
  \label{tab:hessian}
  \setlength{\tabcolsep}{3pt}   
  \small                        
  \begin{tabular}{
      l                         
      S[table-format=3.0]       
      S[table-format=4.2]       
      S[table-format=2.2]       
      S[table-format=1.2]       
  }
    \toprule
    \textbf{Dataset} & {\%} & {Hessian} & {Gap} & {Gap (\%)} \\
    \midrule
    \multirow{8}{*}{\texttt{MNIST}}
      & 100 & 7581.73 &  0.00 & 0.00 \\
      &  80 & 7584.56 &  2.83 & 0.04 \\
      &  60 & 7582.29 &  0.56 & 0.01 \\
      &  40 & 7577.85 &  3.88 & 0.05 \\
      &  20 & 7574.86 &  6.87 & 0.09 \\
      &  10 & 7572.19 &  9.54 & 0.16 \\
      &   5 & 7571.32 & 10.41 & 0.14 \\
      &   1 & 7570.16 & 11.57 & 0.15 \\
    \midrule
    \multirow{8}{*}{\texttt{FashionMNIST}}
      & 100 & 700.59 & 0.00 & 0.00 \\
      &  80 & 700.71 & 0.12 & 0.02 \\
      &  60 & 700.81 & 0.22 & 0.03 \\
      &  40 & 700.13 & 0.46 & 0.07 \\
      &  20 & 699.59 & 1.03 & 0.15 \\
      &  10 & 699.29 & 1.30 & 0.19 \\
      &   5 & 699.08 & 1.51 & 0.22 \\
      &   1 & 698.97 & 1.62 & 0.23 \\
    \midrule
    \multirow{8}{*}{\texttt{CIFAR10}}
      & 100 & 4184.39 &  0.00 & 0.00 \\
      &  80 & 4184.17 &  0.22 & 0.01 \\
      &  60 & 4177.87 &  6.52 & 0.16 \\
      &  40 & 4174.24 & 10.15 & 0.24 \\
      &  20 & 4169.53 & 14.86 & 0.36 \\
      &  10 & 4168.81 & 15.58 & 0.38 \\
      &   5 & 4167.86 & 16.53 & 0.40 \\
      &   1 & 4167.77 & 16.62 & 0.40 \\
    \bottomrule
  \end{tabular}
\end{table}

As shown, across all three benchmarks the absolute curvature magnitude (rows ``\textit{Hessian}'') changes by less than $0.2\%$ between the full data and the $1\%$ subset, indicating that even small but diverse samples capture the Hessian pattern.

\end{document}